\documentclass[11pt]{article}

\usepackage{algorithm}
\usepackage{algpseudocode}  
\usepackage[colorlinks=true,linkcolor=blue,citecolor=red,urlcolor=green]{hyperref}
\usepackage{amsfonts,amssymb}
\usepackage{comment}
 
\usepackage{titlesec}
\usepackage{tikz}
\usepackage{bbding}
\usepackage{pifont}
\usepackage{titletoc}
\usepackage{hhline}
\usepackage{amsmath}
\usepackage{listings}
\usepackage{verbatim}
\usepackage{makecell}
\usepackage{tabu}
\usepackage{graphicx}
\usepackage{enumitem}
\setlist[itemize]{leftmargin=*}
\setlist[enumerate]{leftmargin=*}
\usepackage{mathrsfs}
\usepackage{setspace}
\usepackage{natbib}
\usepackage{bbm}
\usepackage{mathtools}
\usepackage{titletoc}
\usepackage{geometry} 
\usepackage{threeparttable}
\usepackage{url} 
\usepackage{subcaption}

 \newcommand{\nn}{\nonumber}
\newcommand{\explain}{\nn&\quad\blacktriangleright}

\usepackage{listings} 

\usepackage[T1]{fontenc}
\usepackage[utf8]{inputenc}
 
\usepackage{amsthm}
\usepackage{xcolor, tikz}

\newtheorem{lem}{Lemma}
 
\newtheorem{rem}{Remark}

\usepackage{bm}
 
\DeclareMathOperator{\sign}{sign}

\DeclareMathOperator{\diag}{diag}

\DeclareMathOperator{\polylog}{polylog}

\newtheorem{theorem}{Theorem}   
 
\newtheorem{assumption}{Assumption}

\newcommand{\calL}{\mathcal{L}}
\newcommand{\calM}{\mathcal{M}}
\newcommand{\calN}{\mathcal{N}}

\newcommand{\calR}{\mathcal{R}}
\newcommand{\calS}{\mathcal{S}}

\newif\ifjrnotes
\jrnotestrue   

\ifjrnotes
  \newcommand{\jrnote}[1]{%
    {\color{teal}\footnotesize[JR: #1]}%
  }
\else
  \newcommand{\jrnote}[1]{}
\fi

\title{Finite-Sample Performance of Gradient Descent in Logistic Regression with Gaussian Design}

\author{Junren Chen\thanks{Department of Statistics, Columbia University.} \and Arya Mazumdar\thanks{Hal{\i}c{\i}o\u{g}lu Data Science Institute, UC San Diego.}  
}
\date{\today}

\begin{document}
\maketitle
 
\begin{abstract}
 We consider the parameter estimation problem in logistic regression with Gaussian design: the estimation of a fixed unknown parameter $\theta^*\in \mathbb{R}^d$ ($\|\theta^*\|_2\ge 1$) from $n$ i.i.d. samples $\{(x_i,y_i)\}_{i=1}^n$, where $x_i\sim N(0,I_d)$ and $y_i|x_i \sim {\rm Bernoulli}(1/(1+\exp(-x_i^\top \theta^*)))$. Our main aim is to characterize the finite-sample estimation performance and convergence behavior of gradient descent (GD) on the maximum likelihood objective (i.e., the logistic loss). Under small $O(1)$ stepsize  and $0$ initialization, we show that GD   linearly converges to a small neighborhood of $\theta^*$  achieving an $\ell_2$ error of order $O(\sqrt{\|\theta^*\|_2^5d/n})$. This substantially goes beyond existing theoretical results that lack non-asymptotic estimation error rate and exhibit much slower parameter convergence. We also establish a {\it faster} local linear convergence to the same statistical error under a large $\Theta(\|\theta^*\|_2)$ stepsize. The main technical component is to show that the gradient of the logistic loss satisfies a certain approximate invertibility condition (AIC). To that end, we uniformly control the deviation of the gradient from its population counterpart by covering and peeling arguments, and then show that the population GD is a contraction by a delicate  analysis based on the eigenvalues of population Hessian matrices.  Finally, we build upon the recent work \cite{matsumoto2025learning} and devise a novel efficient estimator that attains a sharper rate in high dimensions. This indicates that the existing non-asymptotic guarantees exhibit sub-optimal dependence on $\|\theta^*\|_2$, and that in many regimes $\Theta(\sqrt{\|\theta^*\|_2d/n})$ is the tight estimation error rate. Numerical examples are provided to corroborate our theoretical results.  
\end{abstract}
 \section{Introduction}
Logistic regression is a foundational model in statistics, optimization and machine learning. Given $\{(x_i,y_i)\}_{i=1}^n\subset \mathbb{R}^d\times\{0,1\}$ with binary responses, it assumes the existence of some underlying parameter $\theta^*\in \mathbb{R}^d$ such that the data  follows
$
    y_i|x_i\sim {\rm Bernoulli}(s(x_i^\top \theta^*)),~i\in[n] $
with the sigmoid function $s(a) = 1/(1+e^{-a})$, meaning that $\mathbb{P}(y_i=1)=1/(1+e^{-x_i^\top \theta^*})$ and $\mathbb{P}(y_i=0)=1/(1+e^{x_i^\top \theta^*})$.  
A standard estimator of the true parameter $\theta^*$ is the maximum likelihood
estimator  (MLE) \cite{mccullagh2019generalized} that minimizes the logistic loss function \footnote{Throughout this paper, we adopt the convention that the responses $\{y_i\}_{i=1}^n$ are $\{0,1\}$-valued. One can readily transfer our development to the (equivalent) convention of $\{-1,1\}$-valued responses $\{\tilde{y}_i\}_{i=1}^n$ (e.g., \cite{hsu2024sample,matsumoto2025learning,wu2023implicit}) by letting $y_i = \frac{1}{2}(\tilde{y}_i+1)$, under which the logistic loss reads $L(\theta)=\frac{1}{n}\sum_{i=1}^n \log(1+e^{-\tilde{y}_ix_i^\top \theta})$.} 
\begin{align*}
    L(\theta) =  \frac{1}{n}\sum_{i=1}^n \big(\log(e^{x_i^\top \theta}+1)-\ y_ix_i^\top \theta\big),
\end{align*}
given by $\hat{\theta}_{mle}=\textrm{arg}\min_{\theta\in \mathbb{R}^d}L(\theta)$. Since MLE does not admit a closed form, one typically can minimize $L(\theta)$ via gradient descent (GD) with some stepsize $\eta$,
\begin{align}
    \label{gd1}
    \tag{GD}
    \theta_{t+1} = \theta_t-\eta \nabla L(\theta_t),\quad t=0,1,\cdots,
\end{align}
 where $\nabla L(\theta) = \frac{1}{n}\sum_{i=1}^n(s(x_i^\top \theta)-y_i)x_i$. The problems of {\it estimation} and {\it optimization} have been widely studied in logistic regression.

  The estimation problem is of interests to statisticians and concerned with the performance of a certain estimator in approximating the underlying $\theta^*$. Perhaps a more fundamental question is to determine the  minimax optimal error rate. To answer these questions with an explicit error rate depending on the model parameters $(n,d,\|\theta^*\|_2)$,\footnote{Here, $\|\theta^*\|_2$ indicates how likely $y_i$ is flipped to $-\sign(x_i^\top \theta^*)$
  and hence represents the signal-to-noise ratio (SNR).} existing works assume i.i.d. samples where $x_i$'s are drawn from a specific distribution \cite{kuchelmeister2024finite,chardon2024finite,hsu2024sample,matsumoto2025learning,kakade2011efficient}.

From an optimization perspective,
  a central   question is to understand the capacity of (\ref{gd1}) in minimizing $L(\theta)$. In light of the convexity of $L(\theta)$, traditional optimization theory \cite[Sections 1, 2]{nesterov2018lectures}  yields the global convergence of (\ref{gd1}) under small enough $\eta$, guaranteeing that $L(\theta_t)$ decreases monotonically. The works \cite{soudry2018implicit,ji2019implicit} provided fine-grained characterizations under small stepsizes. A more recent line of works studied large stepsize GD in logistic regression \cite{wu2023implicit,wu2024large,zhang2025gradient,tyurin2025logistic}, referred to as the edge of stability regime  \cite{cohen2021gradient}.

  This paper is concerned with both the estimation and optimization in logistic regression. Following \cite{kuchelmeister2024finite,chardon2024finite,hsu2024sample,matsumoto2025learning} we assume that the covariates $x_1,\cdots,x_n$ are i.i.d. standard Gaussian vectors 
  and focus on the estimation of the true parameter $\theta^*$. %

 \begin{assumption}
  \label{gaussian}
   $\{(x_i,y_i)\}_{i=1}^n$ are i.i.d. samples with $x_i\sim N(0,I_d)$, and when conditioning on $x_i$, $y_i\sim{\rm Bernoulli}(s(x_i^\top \theta^*))$ for some $\theta^*\in \mathbb{R}^d.$
\end{assumption}

\subsection{Related Works}
As it is unrealistic to cover the vast literature  of logistic regression,  we shall focus on two lines of works that are closest to our paper.  

\noindent
\textbf{Parameter estimation.} The most relevant works are \cite{kuchelmeister2024finite,chardon2024finite} that derived the non-asymptotic error rates for MLE under Assumption \ref{gaussian}.  Kuchelmeister and Van de Geer \cite{kuchelmeister2024finite} first established the finite-sample error rate of minimizing $L(\theta)$ under a Probit model. It was proved \cite[Theorem 2.1.1]{kuchelmeister2024finite} that a constrained variant of MLE achieves $\ell_2$ error of $O(\sqrt{\|\theta^*\|_2^3\frac{d\log n}{n}})$ whenever $n\gtrsim \|\theta^*\|_2 d\log n$. More recently, Chardon et al. \cite{chardon2024finite} established similar guarantee for MLE    with refinement, showing \cite[Theorem 1, Equation (19)]{chardon2024finite} that MLE exists and achieves \begin{align}
    \|\hat{\theta}_{\rm mle}-\theta^*\|_2\lesssim \sqrt{\frac{\|\theta^*\|_2^3d}{n}}\label{mlerate}
\end{align} whenever $n\gtrsim \|\theta^*\|_2d$. On a relevant passing note,  \cite{candes2020phase} derived the phase transitions for the existence of MLE in high-dimensional asymptotics.

  Another two closely related works \cite{hsu2024sample,matsumoto2025learning} studied the estimation of parameter direction $r^*:=\frac{\theta^*}{\|\theta^*\|_2}$ through algorithms other than MLE. Under $\|\theta^*\|_2\gtrsim 1$, \cite{hsu2024sample} determined the minimax error rate as 
$\tilde{\Theta}(\max\big\{\sqrt{\frac{d}{n\|\theta^*\|_2}},\frac{d}{n}\big\})$.
 This notably exhibits a transition to the classical halfspace learning result \cite{long1995sample} as $\|\theta^*\|_2\to \infty$. While computational tractability was not a consideration in the 
 the upper bound in \cite{hsu2024sample} (it involves nonconvex  minimization routines),  more recently \cite{matsumoto2025learning} addressed this   by showing that a fast iterative algorithm attains the near optimal direction estimation error rate. 

  We emphasize all the above results study   Gaussian covariates (cf. Assumption \ref{gaussian}) as a benchmark. Extension to sub-Gaussian covariates is highly challenging in general  \cite{chardon2024finite,candes2020phase}.

\noindent
\textbf{Convergence of (\ref{gd1}).} The optimization path of (\ref{gd1}) has been characterized \cite{soudry2018implicit,ji2019implicit}   under small stepsizes.  More recently larger stepsizes have been extensively studied, with positive results under separable data \cite{wu2023implicit,wu2024large,wu2025large,zhang2025minimax} while negative result under non-separable data \cite{meng2024gradient,meng2025gradient}. These results typically work with   data obeying  separability \cite{soudry2018implicit,wu2023implicit,wu2024large,zhang2025minimax,axiotis2023gradient} and bounded covariates $\|x_i\|_2\le 1$ \cite{ji2019implicit,wu2023implicit,wu2024large}.  Yet, such optimization works do not touch on the estimation problem, i.e., how well GD estimates the true parameter $\theta^*$ (indeed some of these works do not assume existence of $\theta^*$).  Also, the convergence in parameter, if obtained, is often very slow, e.g., the direction convergence rates of  \cite{soudry2018implicit,ji2019implicit} are no faster than $O(\frac{1}{\log t})$.

  \cite{kalai2009isotron,kakade2011efficient} studied  algorithms analogous to perceptron, called GLM-tron and Isotron, with the GLM-tron reducing to (\ref{gd1}) with $\eta=1$ when applied to logistic regression.   Both works considered  covariates independently drawn from the unit $\ell_2$-ball, and that the true parameter is bounded $\|\theta^*\|_2=O(1)$. While these methods appear more versatile, the closest result \cite[Theorem 1]{kakade2011efficient} requires $\tilde{O}(\sqrt{n})$ iterations to reach an   error rate no faster than $O(n^{-1/4})$. 

\subsection{Our Contributions}


Despite the huge literature, a few questions remain open. As a   practically relevant algorithm, how well does (\ref{gd1}) estimate $\theta^*$? Under Gaussian design, are there stronger convergence behaviors beyond existing results (e.g., \cite{ji2019implicit,kakade2011efficient,wu2023implicit})? Statistically, is the non-asymptotic error rate of the MLE  in (\ref{mlerate}) optimal? If not, what is the minimax optimal error rate then?

We answer these questions by establishing the following results: 
\begin{itemize}
    \item Under a small stepsize $\eta\in(0,8)$, (\ref{gd1}) with $\theta_0=0$ linearly converges to $\theta^*$ and achieves an error rate of $O(\sqrt{\|\theta^*\|_2^5d/n})$ (cf. Theorem \ref{thm:small}). Both the estimation error rate and fast linear convergence exceed existing results for GD in logistic regression (cf. Remarks \ref{rem:estimation}, \ref{rem:convergence}). 
    \item Under a large stepsize $\eta\asymp\|\theta^*\|_2$, (\ref{gd1}) locally achieves an acceleration--a faster linear convergence  with a sharper per-iteration contraction factor (cf. Theorem \ref{thm:big}). 
    \item A novel efficient estimator   achieves an error rate that is   sharper than the best known rate $O(\sqrt{\|\theta^*\|_2^3d/n})$ achieved by MLE in high dimensions (cf. Theorem \ref{thm:initial}). This indicates that existing estimation rates are sub-optimal in terms of the dependence on $\|\theta^*\|_2$, and that the minimax optimal rate is $\Theta(\sqrt{\|\theta^*\|_2d/n})$ in some regimes (cf. Remarks \ref{rem:improve}, \ref{rem:sharp}).   
\end{itemize}

\noindent
\textbf{Notation.}
 We denote the gradient of the logistic loss by
$
h_{\theta^*}(\theta)
:=
\frac{1}{n}\sum_{i=1}^n\big(s(x_i^\top\theta)-y_i\big)x_i
=\nabla L(\theta).
$ 
Let $\|u\|_2$ be the $\ell_2$ norm of $u$. Define the unit sphere $\mathbb{S}^{d-1}=\{u\in\mathbb{R}^d:\|u\|_2=1\}$ and unit ball $\mathbb{B}_2^d=\{u\in\mathbb{R}^d:\|u\|_2\le 1\}$. For $r>0$ and $\theta_0\in\mathbb{R}^d$, let $\mathbb{B}_2^d(r)=r\mathbb{B}_2^d$ and $\mathbb{B}_2^d(\theta_0;r)=\theta_0+r\mathbb{B}_2^d$. Let $\|M\|_{\mathrm{op}}$ denote the operator norm of a matrix $M$.
For a random variable $X$, define the sub-Gaussian and sub-exponential norms by
$\|X\|_{\psi_2}:=\inf\{K>0:\mathbb{E}[\exp(X^2/K^2)]\le 2\},~
\|X\|_{\psi_1}:=\inf\{K>0:\mathbb{E}[\exp(|X|/K)]\le 2\}$, respectively. 
We write $T_1=O(T_2)$ (or $T_1\lesssim T_2$) to denote $T_1\le C T_2$, and write $T_1=\Omega(T_2)$ (or $T_1\gtrsim T_2$) to denote $T_1\ge C T_2$. We also write $T_1=\Theta(T_2)$ (or $T_1\asymp T_2$) if both $T_1=O(T_2)$ and $T_1=\Omega(T_2)$ hold. We use $\tilde{O}(\cdot)$, $\tilde{\Omega}(\cdot)$, $\tilde{\Theta}(\cdot)$ to hide log factor in the sample size $n$. More notation will be introduced as needed. 

\noindent
\textbf{Overview.} We provide the main results of GD and an overview of the proof techniques in Section \ref{sec:gd}.    A novel estimator that achieves sharper statistical error rate is developed in Section \ref{sec:sharper}. We provide numerical examples in Section \ref{sec:numeric} to corroborate our theory, and give concluding remarks to close the paper in Section \ref{sec:conclude}. The complete proofs of Theorems \ref{thm:small}, \ref{thm:big}, \ref{thm:initial} appear in Appendices \ref{app:proofsmall}, \ref{app:proofbig}, \ref{app:proofsharper}, with supporting lemmas provided in Appendix \ref{app:lemma}.

\section{Gradient Descent}\label{sec:gd}
 
We assume that
$\|\theta^*\|\ge 1$ throughout this paper (it can be relaxed to $\|\theta^*\|_2\ge c$ for some small constant $c>0$). Note that we cover the strong signal regime with $\|\theta^*\|\gg 1$ as in   recent works \cite{hsu2024sample,kuchelmeister2024finite,chardon2024finite,matsumoto2025learning}.
\begin{assumption}\label{normge1}
    $\theta^*\in \mathbb{R}^d$ is an unknown fixed parameter satisfying $\|\theta^*\|_2\ge 1$. 
\end{assumption}

  We   formalize the gradient descent in the following Algorithm \ref{alg:gd}, which is also a specific instance of GLM-tron \cite{kakade2011efficient}. Due to the simplicity this has been one of the most commonly used algorithm in logistic regression, and we study under Gaussian designs the following question: 
  \[\textit{What are the finite-sample estimation performance and convergence behavior of Algorithm~\ref{alg:gd}?}\]


\begin{algorithm}[ht!]   
\caption{Gradient Descent for Logistic Regression}\label{alg:gd}
\begin{algorithmic}[1]   
  \Require data $\{(x_i,y_i)\}_{i=1}^n$, initialization $\theta_0$, stepsize $\eta$
 
  \For{$t = 0,1,2,\cdots$} 
   \begin{align*}
        \theta_{t+1} = \theta_t-\eta \cdot\frac{1}{n}\sum_{i=1}^n (s(x_i^\top \theta_t)-y_i)x_i
    \end{align*}
  \EndFor

  \Ensure $\{\theta_t\}_{t\ge 0}$
\end{algorithmic}
\end{algorithm}

While a large body of convergence results exist, little is known about the estimation performance of Algorithm \ref{alg:gd}. We shall see that our results  are novel to the literature in many ways.

\subsection{Small Stepsize}
Traditional convex optimization theory \cite{nesterov2018lectures} guarantees that Algorithm \ref{alg:gd} converges globally under $\eta\le 2/\beta$, where the smoothness  parameter $\beta:=\sup_{\theta\in \mathbb{R}^d}\|\nabla^2L(\theta)\|_{op}=\sup_{\theta\in \mathbb{R}^d}\|\frac{1}{n}\sum_{i=1}^n s'(x_i^\top \theta)x_ix_i^\top \|_{op}$. By $s'(a)\le s'(0)=\frac{1}{4}$ we obtain $\beta = \|\frac{1}{4n}\sum_{i=1}^n x_ix_i^\top \|_{op}$, and 
under Gaussian designs with $n\gtrsim d$  one has $\beta \approx \frac{1}{4}$ (cf. Lemma \ref{lem:gaubound}). Therefore GD with any fixed $\eta\in(0,8)$  converges globally and achieves a monotonically decreasing logistic loss, referred to as the  stable regime.

But how close are the iterates to the true parameter $\theta^*$? 
In the same stable regime we answer this question by the following result (Theorem \ref{thm:small}). Note that we put $\eta\in[0.1,7.9]$ to avoid the dependence of the contraction factor on $\eta$, but it can be observed from our proof that, under $\eta\in(0,8)$, we obtain a contraction factor of $1- \tilde{c}\|\theta^*\|_2^{-3}\eta(8-\eta)$ for some universal constant $\tilde{c}$.
Also, our result is essentially a {\it global} convergence (although we consider $\theta_0=0$ for concreteness). We show that $\{\theta_t\}_{t\ge 0}$ {\it linearly} converges to a small radius-$O(\sqrt{\|\theta^*\|_2^5d/n})$ $\ell_2$ ball surrounding $\theta^*$ within $O(\|\theta^*\|_2^3\log\frac{n}{d})$ iterations, whenever $n=\tilde{\Omega}_{\|\theta^*\|_2}(d)$. 

\begin{theorem}[Small stepsize GD] \label{thm:small} 
Under Assumptions \ref{gaussian}--\ref{normge1}, for any $\eta\in[0.1,7.9]$, there exist some universal constants $C,c,\tilde{C}>0$ such that the following holds.  If 
\begin{align}
    n\ge C\Big(\|\theta^*\|_2^6 d\log n + \|\theta^*\|_2^6\log\|\theta^*\|_2\Big),\label{samcom_small}
\end{align}
  then with probability at least $1-2e^{-d}$, $\{\theta_t\}_{t\ge 0}$ generated by Algorithm \ref{alg:gd} with $\theta_0=0$ and stepsize $\eta$ satisfies
\begin{align}\label{smallconverge}
    \|\theta_t-\theta^*\|_2\le \bigg(1-\frac{c}{\|\theta^*\|_2^3}\bigg)^t\|\theta^*\|_2+ \tilde{C}\sqrt{\frac{\|\theta^*\|_2^5d}{n}},\quad \forall t=0,1,\cdots.
\end{align} 
\end{theorem}
To position our result we provide two remarks from the angles of estimation and optimization. 
\begin{rem} \label{rem:estimation}
     To our best knowledge, the only existing estimation guarantee of GD with logistic loss is \cite[Theorem 1]{kakade2011efficient} for bounded designs (i.e., $\|x_i\|_2\le 1$), and our Theorem \ref{thm:small} provides the first non-asymptotic estimation error rate to GD in logistic regression with Gaussian designs.  We believe that there are   major differences between Gaussian designs and bounded designs, for instance, the bounded designs bypass the curse of dimensionality and as a result $d$ does not appear in their error rate. Moreover, a naive comparison finds that our Theorem \ref{thm:small} offers the following improvements: (i) Our error decay rate $O(n^{-1/2})$ is faster than their $O(n^{-1/4})$; (ii) To attain the final error rates our iteration complexity is $\tilde{O}(\|\theta^*\|_2^3)$ while theirs reads $\tilde{O}(\|\theta^*\|_2\sqrt{n})$.
\end{rem}

\begin{rem}  \label{rem:convergence}
Theorem \ref{thm:small}  is also interesting from the optimization aspect in terms of the fast linear convergence in parameter. Specifically, recent works \cite{soudry2018implicit,ji2019implicit,nacson2019convergence} also established the convergence of $\theta_t/\|\theta_t\|_2$ to a specific point, but the convergence rates, such as $O(\frac{\ln \ln t}{\ln t})$ from \cite[Theorem 1]{ji2019implicit}, are much slower than ours.  More are known about convergence in loss, such as the decreasing logistic loss with rate $E(\theta_t):=L(\theta_t)-\min_{\theta}L(\theta)=O(\frac{1}{t})$ \cite{nesterov2018lectures}, a comparable   $E(\theta_t)=O(\frac{\log ^2t}{t})$ of recent work \cite{ji2019implicit}, and the more recent linear convergence $E(\theta_t)<0.1\min_{\theta}L(\theta)+\varepsilon$ whenever $t=\Omega(\log\frac{n}{\varepsilon})$   \cite{axiotis2023gradient}. In light of a standard estimate $\calL(\theta)-\calL(\theta^*)\le \frac{1}{8}\|\theta-\theta^*\|_2^2$ (where $\calL(\theta)=\mathbb{E}L(\theta)$ is the population risk), our Equation (\ref{smallconverge}) immediately yields the linear convergence in excess risk
\begin{align}\label{gotorisk}
    \calL(\theta_t)-\calL(\theta^*) \le \bigg(1-\frac{c'}{\|\theta^*\|_2^3}\bigg)^t\|\theta^*\|_2^2 + C'\frac{\|\theta^*\|_2^5d}{n}, \quad \forall t=0,1,\cdots.
\end{align}
We expect that  $\calL(\theta_t)-\calL(\theta^*)$ is comparable to $L(\theta_t)-\min_{\theta}L(\theta)$ up to some concentration error terms and $\calL(\theta_{mle})-\calL(\theta^*)=O(d/n)$ \cite[Theorem 1]{chardon2024finite}. Therefore, our linear convergence in risk is comparable to \cite{axiotis2023gradient} and faster than \cite{nesterov2018lectures,ji2019implicit}. 
\end{rem}
Continuing from Remark \ref{rem:convergence}, we however emphasize that these convergence results are established under different assumptions and therefore  a close comparison may not be possible. In particular, the aforementioned prior results do not rely on Gaussian data, but rather some of them assume bounded designs and/or separable data (possibly with a margin) \cite{axiotis2023gradient,ji2019implicit,soudry2018implicit,nacson2019convergence}. We believe that all these results are complementary to each other. 

\subsection{Large Stepsize}
Due to the ubiquitousness of large stepsize GD in machine learning optimization   \cite{cohen2021gradient}, large stepsize GD for logistic regression receives significant recent interests \cite{wu2023implicit,wu2024large,zhang2025minimax,wu2025large}. Our second main result, the following Theorem \ref{thm:big}, establishes the local linear convergence of GD with a stepsize that can be much larger than $2/\beta\approx 8$.  

\begin{theorem}[Large stepsize GD]
    \label{thm:big}
    Under Assumptions \ref{gaussian}--\ref{normge1}, there exist some universal constants $c_0,c,C,\tilde{C}>0$ such that the following holds. Fix any $\eta>0$ such that
    \begin{align}\label{largestepcon}
        \frac{c}{\|\theta^*\|_2^2q'(\|\theta^*\|_2)}<\eta<\frac{2-c\|\theta^*\|_2^{-2}}{m(\|\theta^*\|_2)}
    \end{align}
    where $m(\tau)=\mathbb{E}_{g\sim N(0,1)}[s'(\tau g)]$,  $q(\tau) = \mathbb{E}_{g\sim N(0,1)}[s(\tau g)g] $ and $q'(\tau)=\mathbb{E}_{g\sim N(0,1)}[s'(\tau g)g^2]$,
    if
    \begin{align}\label{samplebig}
        n\ge C\bigg(\|\theta^*\|_2^6d\log n+\|\theta^*\|_2^7d\bigg),
    \end{align}
    then with probability at least $1-2e^{-d}$, $\{\theta_t\}_{t\ge 0}$ generated by Algorithm \ref{alg:gd} with $\theta_0\in \mathbb{B}_2^d(\theta^*;\frac{c_0}{\|\theta^*\|_2})$ and stepsize $\eta$  satisfies 
    \begin{align}
         \|\theta_t-\theta^*\|_2\le \bigg(1-\frac{c}{\|\theta^*\|_2^2}\bigg)^t\frac{c_0}{\|\theta^*\|_2} + \tilde{C}\sqrt{\frac{\|\theta^*\|_2^5d}{n}},\quad \forall t=0,1,\cdots.
    \end{align}
    Moreover, a specific stepsize $
        \eta=\frac{1}{m(\|\theta^*\|_2)}$ satisfies (\ref{largestepcon}). 
\end{theorem}

In light of $m(\tau)\asymp 1/\tau$ and $q'(\tau)\asymp 1/\tau^{3}$ for $\tau\ge 1$ (cf. Lemma \ref{lem:separa}), the stepsize in Equation (\ref{largestepcon}) is of order $\|\theta^*\|_2$, which is  hence much larger than    $\eta\in(0,8)$ (the stable regime) when $\|\theta^*\|_2\gtrsim 1$ (e.g., the specific $\eta=1/m(\|\theta^*\|_2)$ already exceeds $8$ under $\|\theta^*\|_2\ge 2$). Under $\|\theta^*\|_2\gg 1$, the large stepsize offers a notable acceleration, achieving a per-iteration contraction factor   $1-\Theta(\|\theta^*\|_2^{-2})$ that improves on $1-\Theta(\|\theta^*\|_2^{-3})$ in Theorem \ref{thm:small}. Yet the costs are that an accurate $\theta_0$   and a slightly higher sample complexity are needed, and that the stepsize $\eta$ requires some tuning according to $\|\theta^*\|_2.$ The estimation error rate remains the same.

Similar to Equation (\ref{gotorisk}) in Remark \ref{rem:convergence}, we also have a linear convergence in population risk 
\begin{align}\label{gotorisk2}
    \calL(\theta_t)-\calL(\theta^*) \le \bigg(1-\frac{c'}{\|\theta^*\|_2^2}\bigg)^t\frac{c_0'}{\|\theta^*\|_2^2} + C'\frac{\|\theta^*\|_2^5d}{n}, \quad \forall t=0,1,\cdots.
\end{align}
We now put our result in perspective. 
\begin{rem}\label{rem:wu}
    \cite{wu2023implicit} first showed that large stepsize GD  minimizes logistic loss and attains $E(\theta_t)=L(\theta_t)-\min_{\theta}L(\theta)=O(1/t)$, then \cite{wu2024large} provided finer optimization-theoretic characterizations and established a faster convergence $E(\theta_t)=\tilde{O}(1/t^2)$ (cf. their Corollary 2). See also further developments using adaptive stepsizes \cite{zhang2025minimax} or in a regularized setting \cite{wu2025large}. While  the faster convergence is commonly uncovered as a benefit of large stepsize,
we emphasize that  the settings in this line of works and in our paper are essentially different: \cite{wu2023implicit,wu2024large,zhang2025minimax,wu2025large} assume bounded covariates   and separable data with margin $\gamma$ (see, e.g., \cite[Assumption 1]{wu2024large}), while we treat Gaussian designs and, under the sample complexity (\ref{samplebig}), our data $\{(x_i,y_i)\}_{i=1}^n$ is provably non-separable (cf. \cite[Theorem 1]{chardon2024finite}). A naive comparison finds that our (\ref{gotorisk2}) is faster than \cite{wu2023implicit,wu2024large} and comparable to \cite{zhang2025minimax} using adaptive large stepsizes.

Concerning large stepsizes, there is an essential difference between separable data and non-separable data \cite{tyurin2025logistic,wu2025large}: for linearly separable data the logistic loss is minimized at infinity, which enables arbitrarily large stepsize (e.g., \cite{wu2024large}); in contrast, under non-separable data (and mild condition) logistic loss admits a finite, unique minimizer, thus the GD iterates are necessarily unstable if $\eta$ exceeds some threshold \cite{hirsch2013differential}. Indeed the behaviors of large stepsize GD in non-separable data are much more intricate, for instance, under $d\ge 2$, it was shown \cite[Theorem 3]{meng2024gradient} that GD with   $\eta =  c \|\nabla^2 L(\theta_{mle})\|_{op}^{-1}$, $c\in(0,1]$ may converge to a stable cycle. Our Theorem \ref{thm:big} complements the negative results of \cite{meng2024gradient} with a positive radius-$\Theta(\|\theta^*\|_2^{-1})$ local convergence under Gaussian designs.

Finally, while the different assumptions may forbid a close comparison, we again emphasize that the finite-sample statistical rate, $O_{\|\theta^*\|_2}(\sqrt{d/n})$, is unique to our paper (see also Remark \ref{rem:estimation}). 
\end{rem}  
\subsection{Technical Overview}
The proofs of Theorems \ref{thm:small}, \ref{thm:big} essentially depart from the related works that we have compared with. Our major technical component is a structured condition --- the \emph{approximate invertibility condition} (AIC) of the gradient $h_{\theta^*}(u)$ --- which implies the linear convergence of the GD iterates to a small neighborhood of the true parameter $\theta^*$, hence simultaneously settling the  estimation and convergence.  We note in passing that recent works \cite{matsumoto2024binary,matsumoto2025learning,chen2024optimal,chen2025unified,friedlander2021nbiht,abdalla2026robust,chen2024one,chen2026instance} have leveraged similar structured conditions to analyze different nonlinear regression problems.

Specifically, our AIC in Theorem \ref{thm:small} states that, for $\eta\in[0.1,7.9]$ and some universal constants $c,C$,
\begin{align}
    \label{aicsmallmain}\tag{AIC}
     \|u-\theta^*-\eta\cdot h_{\theta^*}(u)\|_2 \le \bigg(1-\frac{c}{\|\theta^*\|_2^3}\bigg)\|u-\theta^*\|_2+ C\bigg(\sqrt{\frac{d}{n\|\theta^*\|_2}}+\frac{d}{n}\bigg) 
\end{align} 
holds for all $u\in \mathbb{B}_2^d(2\|\theta^*\|_2)$. Note that the (\ref{aicsmallmain}) uniformly controls how much the actual gradient descent step, $\eta\cdot h_{\theta^*}(u)$, deviates from the ideal invertibility step $u-\theta^*$.\,\footnote{Suppose we were at $u$, then taking a descent step of $u-\theta^*$ brings us to $u-(u-\theta^*)=\theta^*$, which is the desired parameter.} As such the AIC alone is enough to yield the claimed guarantee in Theorem \ref{thm:small} (cf. Appendix \ref{app:aic2con}). Similarly, a ``local'' AIC over $\mathbb{B}_2^n(\theta^*;\frac{c_0}{\|\theta^*\|_2})$ yields the local convergence of Theorem \ref{thm:big}. See (\ref{localaic}) in Appendix \ref{app:localaictov}.

The main technical work lies in the establishment of (\ref{aicsmallmain}) under Gaussian designs. To that end we start with a decomposition 
\[\|u-\theta^*-\eta\cdot h_{\theta^*}(u)\|_2\le \eta\cdot \underbrace{\|h_{\theta^*}(u) - \mathbb{E}[h_{\theta^*}(u)]\|_2}_{:=T_c} + \underbrace{\|u-\theta^* -\eta\mathbb{E}[h_{\theta^*}(u)]\|_2}_{:=T_e}.\]

\noindent
\textbf{Bounding $T_c$.} We start with a further decomposition   $T_c\le T_{c1}+T_{c2}$, where $ T_{c1}:=\big\|\frac{1}{n}\sum_{i=1}^n \big( \big(s(x_i^\top \theta^*)-y_i\big)x_i-\mathbb{E}\big[\big(s(x_i^\top \theta^*)-y_i\big)x_i\big]\big)\big\|_2$ and $ T_{c2}:= \big\|\frac{1}{n}\sum_{i=1}^n \big( \big(s(x_i^\top u)-s(x_i^\top \theta^*)\big)x_i-\mathbb{E}\big[\big(s(x_i^\top u)-s(x_i^\top \theta^*)\big)x_i\big]\big)\big\|_2.$
It turns out that establishing sharp bounds on the two terms requires rather different techniques. For $T_{c1}$, a direct application of the sub-Gaussianity of $\frac{1}{n}\sum_{i=1}^n(s(x_i^\top \theta^*)-y_i)x_i-\mathbb{E}\big[\big(s(x_i^\top \theta^*)-y_i\big)x_i\big]$ does not give the tight bound since this overlooks the fact that the variance of the multiplier, $s(x_i^\top \theta^*)-y_i$, decreases as $\|\theta^*\|_2$ increases. To capture the effect of $\|\theta^*\|_2$ we have to use the Bernstein's inequality under moment conditions along with a covering argument (cf. Appendix \ref{app:bernsteinmoment}). In contrast, the key feature of $T_{c 2}$ is that it becomes smaller for $u$ being close to  $\theta^*$, thus requiring the technique of localization. A standard approach to this end is via symmetrization, Talagrand's contraction and concentration of empirical process (e.g., \cite{ledoux1991probability,bartlett2002rademacher}), while following this route leads to intricacy in the contraction due to the product process structure of $T_{c2}$. We instead return to a more elementary  peeling argument and decompose $\mathbb{B}_2^d(2\|\theta^*\|_2)$ into a small number of shells, over each of which $\|u-\theta^*\|_2$ either roughly remains at a constant (up to a factor of $2$) or is too small to have an impact. We then control $T_{c2}$ over each shell via covering argument. See Appendix \ref{app:pealing} for details.

\noindent
\textbf{Bounding $T_e$.}  This part of analysis is more elementary but remains rather subtle. Setting the stage, we first (i) use Stein's identity \cite{stein1972bound,stein2004use} to write the   population gradient $\mathbb{E}[h_{\theta^*}(u)]$ as $F(u)-F(\theta^*)$ where $F(u)=m(\|u\|_2)u$, and $m(\tau)=\mathbb{E}_{g\sim N(0,1)}[s'(\tau g)]$ is defined as in Theorem \ref{thm:big}, and then (ii) introduce a useful integral representation $F(u)-F(\theta^*)=\int_0^1\nabla F(\theta^*+t(u-\theta^*))(u-\theta^*)\,dt$ (cf. Lemma \ref{lem:integral}). Based on this, the main bulk of techniques lies an eigenvalue analysis of $\nabla F(z)=m(\|z\|_2)I_d+\frac{m'(\|z\|_2)}{\|z\|_2}zz^\top $,  the population Hessian matrix, along with a careful account of how well the linear map $\nabla F(\theta^*)(u-\theta^*)$ approximates $F(u)-F(\theta^*)$. Note that the eigenvalues of $\nabla F(z)$ are
 $q'(\|z\|_2) $
in the   direction of $z/\|z\|_2$ and $
m(\|z\|_2)
$
in the orthogonal space. Due to 
$
m(\tau)\asymp \tau^{-1},
~
q'(\tau)\asymp \tau^{-3}$ for $\tau\ge 1$ (cf. Lemma \ref{lem:separa}), $\nabla F(\|\theta^*\|_2)$ is highly anisotropic when $\|\theta^*\|_2\gg 1$, 
which is exactly what dictates the   different convergence behaviors under small and large stepsizes (cf. Appendices \ref{app:biasthm1} and \ref{app:bias_thm2}).


\section{Sharper Error Rate}\label{sec:sharper}
     In this section, we ask a more fundamental question: \textit{What is the best possible error rate for the estimation of $\theta^*$?} \footnote{Note that the result of this section has been substantially improved by the authors' follow-up work \cite{chen2026minimax}, which nonetheless builds upon some of the ideas proposed here, such as separate estimation of the direction and norm via sample splitting.}
 In particular, we explore to what extent existing non-asymptotic error rates --- including (\ref{mlerate}) for MLE \cite{kuchelmeister2024finite} and our $O(\sqrt{\|\theta^*\|_2^5d/n})$ for GD --- are (sub-)optimal.    It was shown (e.g., \cite{abramovich2016model}) that the scaling of $(n,d)$, namely $\sqrt{d/n}$, is sharp. Therefore it remains to investigate the dependence on $\|\theta^*\|_2$, in which the MLE rate   is sharper than our GD rate. We show that even the MLE rate $O(\sqrt{\|\theta^*\|_2^3d/n})$ due to \cite{chardon2024finite} exhibits sub-optimal dependence on $\|\theta^*\|_2$. In fact, in this section we  devise a computationally efficient estimator that attains a sharper estimation error rate in high dimensions (cf. Remark \ref{rem:improve}). In some regimes, this error rate reduces to $\tilde{O}(\sqrt{\|\theta^*\|_2d/n})$ and is nearly optimal in light of a matching lower bound (cf. Remark \ref{rem:sharp}).

We shall start with a potential caveat of GD and MLE in parameter estimation: both algorithms simultaneously estimate the   direction $\theta^*/\|\theta^*\|_2$ and the norm  $\|\theta^*\|_2$. This could be problematic because  when $\|\theta^*\|_2$ increases, direction estimation becomes easier while norm estimation becomes harder. In the extreme regime with $\|\theta^*\|_2\to\infty$, logistic regression reduces to the halfspace learning problem where norm estimation is not possible \cite{long1995sample,jacques2013robust,matsumoto2024binary,hsu2024sample}. Therefore, under $\|\theta^*\|_2\gg 1$, the norm estimation error dominates the direction estimation error and dictates the final error rate. 

To overcome this limitation, our main idea is to estimate the direction and norm separately. 
For direction estimation we naturally invoke the  
 recent near-optimal direction estimator due to \cite{matsumoto2025learning}, which \emph{does not require knowledge of $\|\theta^*\|_2$} and, intriguingly, is essentially a {\it (sub)gradient descent with a different loss function}, the ReLU loss;   more details are provided in Appendix \ref{app:subgd}. We restate a non-asymptotic direction estimation bound from \cite{matsumoto2025learning} (see Appendix \ref{app:derilem}), which shows an improved estimate under larger $\|\theta^*\|_2.$ In this section, we use $\polylog(n)$ to denote a generic factor bounded by $(\log n)^C$ for some universal constant $C>0$.

\begin{lem}\cite[Corollary 7 \& Theorem 5]{matsumoto2025learning}\label{lem:directionbound}  Under Assumptions \ref{gaussian}--\ref{normge1} and fix some $T_0\ge \log_2\log_2\frac{n}{d}$, if $n\ge C d\polylog n$ with large enough universal constant $C$, then with probability at least $1-e^{-n}$, $\hat{r}_{T_0}$ obtained by running $\tilde{r}_{t+1} =\hat{r}_t-\frac{\sqrt{2\pi}}{n}\sum_{i=1}^{n}\big(\frac{\sign(x_i^\top \hat{r}_t)+1}{2}-y_i\big)x_i,~\hat{r}_{t+1}=\frac{\tilde{r}_{t+1}}{\|\tilde{r}_{t+1}\|_2},~~ t\ge 0$ (or equivalently, Equation (\ref{lloydequa}) below)
with an arbitrary $r_0\in\mathbb{S}^{d-1}$ satisfies 
\begin{align}
    \bigg\|\hat{r}_{T_0} - \frac{\theta^*}{\|\theta^*\|_2}\bigg\|_2 \le \polylog(n) \bigg(\sqrt{\frac{d}{n\|\theta^*\|_2}}+\frac{d}{n}\bigg).\label{directionerror}
\end{align}
\end{lem}

Before moving on to the norm estimate, we mention that we utilize sample splitting in this estimator. Particularly, for some $\nu \in[0.1,0.9]$ such that $\nu n$ is an integer, we utilize $\{(x_i,y_i)\}_{i=1}^{\nu n}$ to obtain the direction estimator $\hat{r}_{T_0}$ and then   we further estimate $\|\theta^*\|_2$ using $\{(x_i,y_i)\}_{i=\nu n+1}^{n}$ (and $\hat{r}_{T_0}$). This ensures some independence and facilitates the analysis, and has been applied to different problems \cite{netrapalli2013phase,LugosiMendelson2021}. Yet we believe that the practical performance of the estimator does not rely on this.

To estimate the norm, we use the remaining samples to construct the so-called ``average'' estimator  
$\hat{v} = \frac{1}{(1-\nu)n}\sum_{i=\nu n+1}^{n} y_ix_i$ \cite{plan2017high,servedio1999pac} whose expectation is given by (cf. Lemma \ref{Ehatv}) 
\begin{align}\label{Ehatvre}
    &v: = \mathbb{E}[\hat{v}] = \mathbb{E}[y_ix_i]= q(\|\theta^*\|_2) \theta^*/\|\theta^*\|_2 
\end{align}
where $q(\tau) = \mathbb{E}[s(\tau g)g]=\mathbb{E}[s'(\tau g)\tau]$ has appeared in Theorem \ref{thm:big}.  
  By Lemma \ref{lem:mono}, $q$ has an inverse on $(0,\infty)$, denoted by $q^{-1}:(0,\frac{1}{\sqrt{2\pi}})\to (0,\infty)$. Hence we seek an estimate of $q(\|\theta^*\|_2)$ and then apply $q^{-1}$ to get an estimate of $\|\theta^*\|_2$. Due to $\|v\|_2=q(\|\theta^*\|_2)$ from (\ref{Ehatvre}), a naive estimate of $q(\|\theta^*\|_2)$ is $\|\hat{v}\|_2$, yet the concentration of $\|\hat{v}\|_2$ about $\|v\|_2$ is under Euclidean norm and necessarily involves dimension $d$. It turns out that a better idea is based on the inner product $\langle v,\theta^*/\|\theta^*\|_2\rangle$,  which leads us to using the optimal direction estimator, $\hat{r}_{T_0},$ as a surrogate of $\theta^*/\|\theta^*\|_2$. Hence we propose to use $\langle \hat{v},\hat{r}_{T_0}\rangle$ to estimate $q(\|\theta^*\|_2)$,  and in turn $q^{-1}(\langle \hat{v},\hat{r}_{T_0}\rangle)$ is our final norm estimator.

  Putting pieces together, we formalize the estimator in Algorithm \ref{alg:spectral}. Our main result in this section provides a non-asymptotic $\ell_2$ error rate achieved by this novel estimator. 

\begin{algorithm}[ht!]   
\caption{An Estimator with Improved Error Rate}\label{alg:spectral}
\begin{algorithmic}[1]   
  \Require data $\{(x_i,y_i)\}_{i=1}^n$, iteration number $T_0$, initialization $\hat{r}_0\in \mathbb{S}^{d-1}$,   parameter $\nu\in(0,1)$ (such that $\nu n$ is an integer)

  \State \textbf{Step 1: Direction estimator}

  \For{$t = 0,1,2,\cdots,T_0-1$} 
   \begin{align}
        \hat{r}_{t+1} = \frac{\hat{r}_t-\frac{\sqrt{2\pi}}{\nu n}\sum_{i=1}^{\nu n}\big(\frac{\sign(x_i^\top \hat{r}_t)+1}{2}-y_i\big)x_i}{\|\hat{r}_t-\frac{\sqrt{2\pi}}{\nu n}\sum_{i=1}^{\nu n}\big(\frac{\sign(x_i^\top \hat{r}_t)+1}{2}-y_i\big)x_i\|_2}\label{lloydequa}
    \end{align}
  \EndFor

   \State \textbf{Step 2: Norm estimator} 

   \State Compute 
   $\hat{v} = \frac{1}{(1-\nu)n}\sum_{i=\nu n+1}^ny_ix_i$
   \State Compute 
   $\hat{\lambda}= q^{-1}(\langle \hat{r}_{T_0},\hat{v}\rangle)$ 
   
  \Ensure $\hat{\theta}= \hat{\lambda}\cdot \hat{r}_{T_0}$
\end{algorithmic}
\end{algorithm}

 \begin{theorem}[Performance of Algorithm \ref{alg:spectral}]\label{thm:initial}
 Under Assumptions \ref{gaussian}--\ref{normge1}, if $
     n\ge \polylog n\big(\|\theta^ *\|_2d+\|\theta^*\|_2^4\big)$ 
 then with probability at least $1-n^{-1}$, $\hat{\theta}$ obtained by Algorithm \ref{alg:spectral} with $T_0\ge \log_2\log_2\frac{n}{d}$, $\hat{r}_0=(1,0,\cdots,0)^\top $ and $\nu\in[0.1,0.9]$ satisfies  
     \begin{align}
         \|\hat{\theta}-\theta^*\|_2 \le \polylog (n)\cdot \bigg(\|\theta^*\|_2^3 \sqrt{\frac{1}{n}} + \sqrt{\frac{\|\theta^*\|_2d}{n}} + \frac{\|\theta^*\|_2^2d}{n}\bigg).\label{fasterrate1}
     \end{align}
 \end{theorem}
 \begin{rem}\label{rem:improve}
     Note that the rate (\ref{fasterrate1}) improves on the best known rate, the $O(\sqrt{\|\theta^*\|_2^3d/n})$ for MLE \cite{kuchelmeister2024finite,chardon2024finite}, as long as  $\polylog(n)\|\theta^*\|_2^3\frac{1}{\sqrt{n}}\le \sqrt{\|\theta^*\|_2^3d/n}$, or equivalently, $d\ge\|\theta^*\|_2^3\polylog(n)$, which covers most high-dimensional regimes of interest. Similarly, it improves on our GD rate, $O(\sqrt{\|\theta^*\|_2^5d/n})$, whenever $d\ge \|\theta^*\|_2\polylog(n)$. 
 \end{rem}
\begin{rem}
\label{rem:sharp}
  Equation (\ref{fasterrate1}) reduces to $\tilde{O}(\sqrt{\|\theta^*\|_2d/n})$ whenever $\|\theta^*\|_2^3 \sqrt{\frac{1}{n}} + \frac{\|\theta^*\|_2^2d}{n}\lesssim\sqrt{\frac{\|\theta^*\|_2d}{n}}$, or equivalently, 
  $
         n\gtrsim \|\theta^*\|_2^3d+\|\theta^*\|_2^5. $ In this regime, the rate $\tilde{O}(\sqrt{\|\theta^*\|_2d/n})$ achieved by Algorithm \ref{alg:spectral} is nearly sharp in light of a matching lower bound $\Omega(\sqrt{\|\theta^*\|_2d/n})$. In fact, suppose that $\tilde{\theta}$ is an estimator for $\theta^*$, then $\tilde{\theta}/\|\tilde{\theta}\|_2$ is an estimator for $\theta^*/\|\theta^*\|_2$ and   therefore performs no better than an existing lower bound  $\|\frac{\tilde{\theta}}{\|\tilde{\theta}\|_2}-\frac{\theta^*}{\|\theta^*\|_2}\|_2\gtrsim \sqrt{\frac{d}{\|\theta^*\|_2n}}$ \cite[Theorem 1]{hsu2024sample}; combining with a known bound $\|\frac{\tilde{\theta}}{\|\tilde{\theta}\|_2}-\frac{\theta^*}{\|\theta^*\|_2}\|_2\le \frac{2\|\tilde{\theta}-\theta^*\|_2}{\|\theta^*\|_2}$ (e.g., \cite[Fact 13]{matsumoto2025learning}), we obtain $\|\tilde{\theta}-\theta^*\|_2\gtrsim \sqrt{\frac{\|\theta^*\|_2d}{n}}$. Therefore, we conclude that the tight error rate in the regime of $n\gtrsim \|\theta^*\|_2^3d+\|\theta^*\|_2^5$ scales as $\sqrt{\|\theta^*\|_2d/n}$. Nonetheless, the  fundamental question of the optimal rate under $n\lesssim \|\theta^*\|_2^3d+\|\theta^*\|_2^5$ remains open.
 \end{rem}
 \begin{rem}\label{rem:exten}
 We mention that Algorithm \ref{alg:spectral} generalizes to high-dimensional sparse case where the true parameter $\theta^*$ is $k$-sparse. To that end, we only need to incorporate a hard thresholding operator into (\ref{lloydequa}); see \cite[Algorithm 1]{matsumoto2025learning}. Under $m=\tilde{\Omega}(\|\theta^*\|_2k+\|\theta^*\|_2^4)$, this sparse counterpart of  Algorithm \ref{alg:spectral} achieves $\tilde{O}(\frac{\|\theta^*\|_2}{\sqrt{n}}+\sqrt{\frac{\|\theta^*\|_2k}{n}}+\frac{\|\theta^*\|_2^2k}{n})$. In contrast, it appears to be much harder to extend Theorems \ref{thm:small}, \ref{thm:big} to the sparse case. 
 \end{rem}

\section{Simulations}\label{sec:numeric}
We provide numerical examples to corroborate our theoretical results. We follow Assumptions \ref{gaussian}, \ref{normge1} to generate $\{(x_i,y_i)\}_{i=1}^n$, and set  $\theta^*=\|\theta^*\|_2r^*$ where $r^*$ is uniformly distributed over $\mathbb{S}^{n-1}$. Each result is averaged over $50$ independent trials.  
 All experiments were implemented using Matlab R2022a on a laptop with an Intel CPU up to 2.5 GHz and 32 GB RAM.

\noindent 
\textbf{Estimation performance.} We first validate the theoretical estimation error rate  $O(\sqrt{\|\theta^*\|_2^5d/n})$ of GD. We run  Algorithm \ref{alg:gd} with $\theta_0=0$, $\eta=4$, and use $\theta_{100}$ as estimator. We test $n\in\{3000,6000,12000,24000\}$ under three settings of $(d,\|\theta^*\|_2)=(200,2),\,(400,2),\,(200,3)$. The log-log plots in Figure~\ref{fig:sub1} validate the predicted $O(n^{-1/2})$ decay rate and confirm that the estimation error increases with both $d$ and $\|\theta^*\|_2$.

\noindent 
\textbf{Convergence under $\eta\in(0,8)$.}
To illustrate the linear convergence predicted by Theorem~\ref{thm:small} for constant stepsizes, we run Algorithm~\ref{alg:gd} from $\theta_0=0$ with $\eta=1$ and $\eta=4$ under the setting
$(n,d,\|\theta^*\|_2)=(5000,200,4)$.
Figure~\ref{fig:sub2} reports the estimation error
$\|\theta_t-\theta^*\|_2$ over the first $200$ iterations on a logarithmic scale.
The approximately straight-line decay in the early iterations is consistent with the linear convergence behavior established in Theorem~\ref{thm:small}.

\noindent
\textbf{Large stepsize.} To illustrate the local acceleration predicted by Theorem~\ref{thm:big}, we run Algorithm~\ref{alg:gd} in a high-signal setting
$(n,d,\|\theta^*\|_2)=(80000,100,8)$.
We initialize the iterates locally around $\theta^*$ by setting
$\theta_0=\theta^*+u$, where $u\sim\mathrm{Unif}(\mathbb{S}^{d-1})$.
We compare the constant stepsize $\eta=4$ with the large stepsize
$\eta=1/m(8)\approx 20.63$.
Figure~\ref{fig:sub3} reports the estimation error over the first $40$ iterations.
The large-stepsize trajectory displays a much steeper initial decay and reaches the statistical error plateau within only a few iterations, consistent with the accelerated local contraction in Theorem~\ref{thm:big}.

\begin{figure}[ht!]
    \centering

    \begin{subfigure}[t]{0.32\textwidth}
        \centering
        \includegraphics[width=\textwidth]{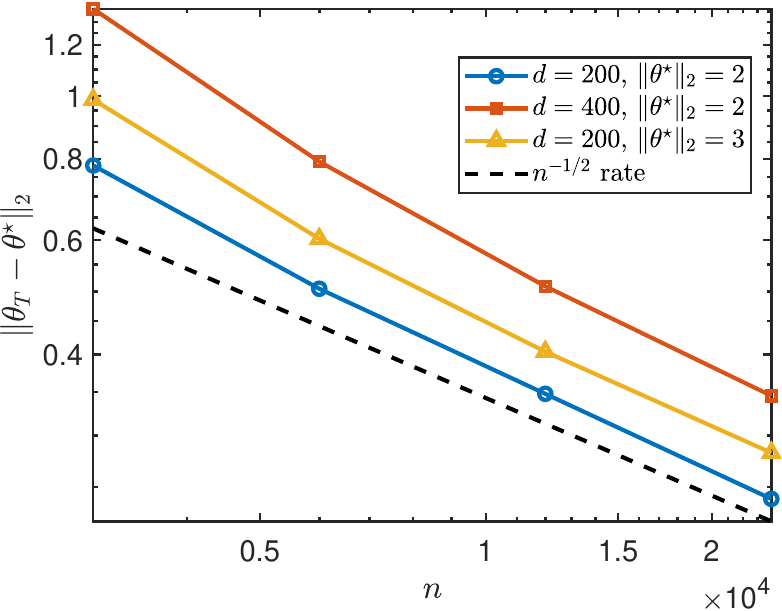}
        \caption{Estimation performance}
        \label{fig:sub1}
    \end{subfigure}
    \hfill
    \begin{subfigure}[t]{0.32\textwidth}
        \centering
        \includegraphics[width=\textwidth]{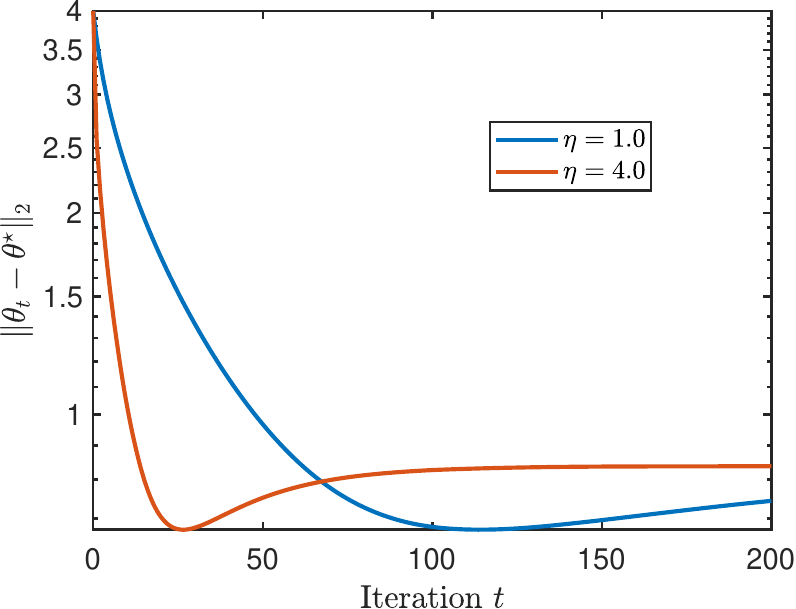}
        \caption{Linear conver. $\eta=1,~4$}
        \label{fig:sub2}
    \end{subfigure}
    \hfill
    \begin{subfigure}[t]{0.32\textwidth}
        \centering
        \includegraphics[width=\textwidth]{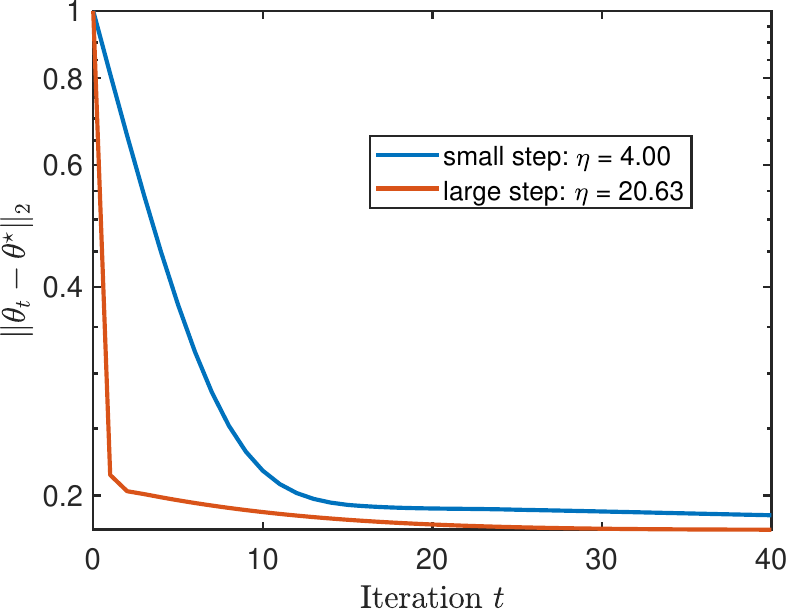}
        \caption{$\eta=4$ v.s. $\eta=20.63$}
        \label{fig:sub3}
    \end{subfigure}

    \caption{Estimation and convergence of GD under Gaussian designs.}
    \label{fig:three_figs}
\end{figure}

\noindent
\textbf{Performance of Algorithm \ref{alg:spectral}.} While the aim of Algorithm \ref{alg:spectral} in this paper is primarily to understand the  minimax optimal rate, we also evaluate its numerical performance. We compare Algorithm \ref{alg:spectral} (without sample splitting) and GD with $\eta=4$ and sufficiently many iterations, which converges to the MLE if it exists. Our numerical results suggest that Algorithm \ref{alg:spectral}   outperforms GD in high dimensions under a moderately large $n$. Due to page limit we defer the discussion to Appendix \ref{app:numalg2}.

\section{Conclusion}\label{sec:conclude}
We provide a finite-sample analysis to gradient descent (GD) for logistic regression with Gaussian designs. We show that GD with $O(1)$ stepsize and $0$ initialization linearly converges to the true parameter $\theta^*$ and achieves error rate $O(\sqrt{\|\theta^*\|_2^5d/n})$. We also show that GD with $\Theta(\|\theta^*\|_2)$ stepsize and a good enough initialization achieves a faster linear convergence and the same estimation performance. The two results are among the first estimation performance bounds to GD in logistic regression and exhibit a convergence rate faster than most existing results. We also devise a novel estimator whose estimation error rate improves on MLE, GD in high dimensions and is nearly sharp  when $n\gtrsim \|\theta^*\|_2^3d+\|\theta^*\|_2^5$. There remain many interesting problems worth further investigations, including extension of our GD results to sparse logistic regression, a finite-sample analysis of GD with adaptive stepsizes, estimator with error rate sharper than Algorithm \ref{alg:spectral}, among many others.

\subsubsection*{Acknowledgment}
The authors used AI for polishing the language only. The work was done while JC was affiliated with Department of Mathematics at the University of Maryland, College Park.

\bibliography{libr}



\appendix

 \section{Proof of Theorem \ref{thm:small} (GD with Small Stepsize)}  \label{app:proofsmall}
 The main ingredient is an approximate invertibility condition (AIC) that controls 
 \begin{align*}
     \|u - \theta^* - \eta\cdot h_{\theta^*}(u)\|_2,\quad \forall u\in \mathbb{B}_2^d(2\|\theta^*\|_2). 
 \end{align*}
 We start with a decomposition 
 \begin{align}\label{aicdecompose1}
     \|u-\theta^*-\eta\cdot h_{\theta^*}(u)\|_2\le \eta\cdot \underbrace{\|h_{\theta^*}(u) - \mathbb{E}[h_{\theta^*}(u)]\|_2}_{:=T_c} + \underbrace{\|u-\theta^* -\eta\mathbb{E}[h_{\theta^*}(u)]\|_2}_{:=T_e},
 \end{align}
 where $T_c$ is a concentration term and $T_e$ is a bias term. 
 \subsection{Bounding $T_c$ (concentration term)}\label{sec:boundTc}
 By substituting $h_{\theta^*}(u)=\frac{1}{n}
 \sum_{i=1}^n(s(x_i^\top u)-y_i)x_i$ and a further decomposition, 
 \begin{align}\nn
     T_c: &= \bigg\|\frac{1}{n}\sum_{i=1}^n \bigg( \big(s(x_i^\top u)-y_i\big)x_i-\mathbb{E}\Big[\big(s(x_i^\top u)-y_i\big)x_i\Big]\bigg)\bigg\|_2\\\nn
     & \le  \bigg\|\frac{1}{n}\sum_{i=1}^n \bigg( \big(s(x_i^\top \theta^*)-y_i\big)x_i-\mathbb{E}\Big[\big(s(x_i^\top \theta^*)-y_i\big)x_i\Big]\bigg)\bigg\|_2 \\\nn
     & + \bigg\|\frac{1}{n}\sum_{i=1}^n \bigg( \big(s(x_i^\top u)-s(x_i^\top \theta^*)\big)x_i-\mathbb{E}\Big[\big(s(x_i^\top u)-s(x_i^\top \theta^*)\big)x_i\Big]\bigg)\bigg\|_2\\:
     &=T_{c1}+T_{c2}. \label{decompconcen} 
 \end{align}
 Conditioning on $x_i$, Assumption \ref{gaussian} ensures $\mathbb{E}[s(x_i^\top \theta^*)-y_i]=0$, hence $T_{c1}$ can be written as 
\begin{align*}
    T_{c1} = \bigg\|\frac{1}{n}\sum_{i=1}^n\big(s(x_i^\top \theta^*)-y_i\big)x_i\bigg\|_2 
\end{align*}
 
\subsubsection{Bounding $T_{c1}$}\label{app:bernsteinmoment}
We shall establish that
\begin{align}\label{Tc1bound}
    T_{c1}\lesssim \sqrt{\frac{d}{n\|\theta^*\|_2}} + \frac{d}{n}
\end{align}
holds with probability at least $1-e^{-d}$. This is a consequence of the following result, where we bound the multiplier process $T_{c1}=\sup_{v\in \mathbb{S}^{d-1}}\frac{1}{n}\sum_{i=1}^n (s(x_i^\top \theta^*)-y_i)x_i^\top v$ via Bernstein's inequality (cf. Lemma \ref{lem:bernstein210}) along with a covering argument.  

 \begin{lem}[Bounding $T_{c1}$ via covering]
There exist universal constants $C,c>0$ such that, for every $\delta\in(0,1)$, with probability at least $1-\delta$,
\[
\left\|
\frac1n\sum_{i=1}^n \big(s(x_i^\top \theta^*)-y_i\big) x_i
\right\|_2
\le
C\left[
\sqrt{
\frac{d+\log(1/\delta)}
{n\|\theta^*\|_2}
}
+
\frac{d+\log(1/\delta)}{n}
\right].
\]
\end{lem}

\begin{proof}
For convenience we let $R:=\|\theta^*\|_2$ and $\xi_i=s(x_i^\top \theta^*)-y_i$,  then we seek to bound $\|\frac{1}{n}\sum_{i=1}^n \xi_ix_i\|_2$. We now construct a $\frac{1}{2}$-net of $\mathbb{S}^{d-1}$, denoted by $N_{1/2}$, such that its cardinality is bounded by $|N_{1/2}|\le 5^d$ (e.g., \cite[Corollary 4.2.13]{vershynin2018high}).  By a standard discretization argument (see, e.g., \cite[Exercise 4.4.2]{vershynin2018high}), 
\begin{align}\label{netbound1}
    \bigg\|\frac{1}{n}\sum_{i=1}^n \xi_ix_i\bigg\|_2\le 2\sup_{v\in N_{1/2}} \frac{1}{n}\sum_{i=1}^n \xi_i x_i^\top v.
\end{align}
Fix any unit vector $v \in \mathbb{S}^{d-1}$ we introduce a shorthand
$Y_i(v):=\xi_i\, x_i^\top v$. In light of (\ref{netbound1}), all that remains is to (i) bound $\frac{1}{n}\sum_{i=1}^n Y_i(v)$ for a fixed $v\in\mathbb{S}^{d-1}$ and then (ii) take a union bound over $N_{1/2}$.

\paragraph{Bounding $\frac{1}{n}\sum_{i=1}^n Y_i(v)$.} Fix $v\in\mathbb{S}^{d-1}$. Recall that $\mathbb E Y_i(v)=0$. Moreover, conditional on $x_i$,
\begin{align}
    \mathbb E[\xi_i^2\mid x_i]& = s(x_i^T\theta^*)^2 + s(x_i^T\theta^*) - 2s(x_i^T\theta^*)^2 \nn
\\&=
s(x_i^\top\theta^*)\big(1-s(x_i^\top\theta^*)\big)
=
s'(x_i^\top\theta^*),\label{Exi2}
\end{align}
where the last equality is due to Lemma \ref{lem:derivative}.
Hence
\[
\mathbb E[Y_i(v)^2]
=
\mathbb E\big[s'(x_i^\top\theta^*)(x_i^\top v)^2\big].
\]

We now bound this variance. Let $e=\theta^*/\|\theta^*\|_2$ and $g_1:=x_i^\top e\sim N(0,1)$, which along with $R=\|\theta^*\|_2$ deliver $x_i^\top \theta^*=Rg_1$. By rotational invariance we can decompose $v$ along $e=\theta^*/\|\theta^*\|_2$ and the orthogonal directions, that is \[v=(v^\top e)e+[v-(v^\top e)e],\quad\textrm{where}~\|v-(v^\top e)e\|_2= \sqrt{1-(v^\top e)^2},\] and then establish    
\begin{align}
    x_i^\top v
=
v^\top e\cdot g_1+\sqrt{1-(v^\top e)^2}\cdot g_2,\quad\textrm{where $g_2\sim N(0,1)$ is independent of $g_1$}\label{decomposexiv}
\end{align}
Then 
\begin{align}\nn
&\mathbb{E}[Y_i(v)^2]=\mathbb E\big[s'(x_i^\top\theta^*)(x_i^\top v)^2\big]
\\\nn
& = \mathbb{E}\Big[s'(Rg_1)\big\{v^\top e\cdot g_1+\sqrt{1-(v^\top e)^2}\cdot g_2\big\}^2\Big]
\\\explain\textrm{by (\ref{decomposexiv})}
\\&=\nn
(v^\top e)^2\mathbb E[s'(Rg_1)g_1^2]
+
\big(1-(v^\top e)^2\big)\mathbb E[s'(Rg_1)]
\\\nn
&\le
\mathbb E[s'(Rg_1)]\\\explain\textrm{by Lemma \ref{lem:V1larger}}\\
&\lesssim \frac{1}{R}\label{varianceestimate}\\
\explain\textrm{by Lemma \ref{lem:separa} and $R=\|\theta^*\|_2\ge 1$ (cf. Assumption \ref{normge1})}
\end{align}
We proceed to estimate higher-order moments, $\mathbb{E}|Y_i(v)|^q$ for integer $q\ge 3$. Similarly to (\ref{varianceestimate}), for any integer $q\ge 3$,  
\begin{align}\nn
    &\mathbb{E}|Y_i(v)|^q = \mathbb{E}|\xi_i x_i^\top v|^q \\
    &\le \mathbb{E}|\xi_i|^2  |x_i^\top v|^{q}\nn\\\nn
    \explain\textrm{by $|\xi_i|=|s(x_i^\top \theta^*)-y_i|\le 1$}\\
    &= \mathbb{E}\big(s'(x_i^\top \theta^*)|x_i^\top v|^q\big)\nn\\\nn
    \explain\textrm{by Equation (\ref{Exi2})}\\
\nn    &=\mathbb{E}\Big(s'(Rg_1)\Big|(v^\top e)g_1+\sqrt{1-(v^\top e)^2}\cdot g_2\Big|^q\Big) \\
    \explain \textrm{by (\ref{decomposexiv}) and $g_1=x_i^\top e$ with $e=\theta^*/\|\theta^*\|_2$}\\
    \nn&\le 2^q\bigg(\mathbb{E}\Big[s'(Rg_1)|g_1|^q\Big] + \mathbb{E}\Big[s'(Rg_1)|g_2|^q\Big]\bigg)\\
    \explain\textrm{by $(T_1+T_2)^q\le 2^q(|T_1|^q+|T_2|^q)$, ~$|v^\top e|\le 1$} \\
    \nn&\le 2^q \bigg(\frac{2q!}{\sqrt{2\pi}}\frac{1}{R^{q+1}} + q!\cdot  \mathbb{E}[s'(Rg_1)]\bigg)\\
    \explain \textrm{by Lemma \ref{lem:Hq-upper}, $g_1\perp g_2$, and $\mathbb{E}|g_2|^q\le q!$}\\
    &\lesssim \frac{2^q q!}{R}\label{qmomentbound}\\
    \explain\textrm{by Lemma \ref{lem:separa} and $\|\theta^*\|_2\ge 1$}
\end{align}
Combining (\ref{varianceestimate}) and (\ref{qmomentbound}), we have validated the conditions in Bernstein's inequality (cf. Lemma \ref{lem:bernstein210}) with $X_i= \frac{Y_i(v)}{n}$, $v\lesssim \frac{1}{nR}$ and $c\lesssim\frac{1}{n}$, and as a result we obtain
\begin{align*}
    \mathbb{P}\bigg(\bigg|\frac{1}{n}\sum_{i=1}^n Y_i(v)\bigg|\lesssim \sqrt{\frac{t}{nR}}+\frac{t}{n}\bigg)\ge 1- 2\exp(-t),\quad \forall t>0.
\end{align*}

\paragraph{Union bound.} We now take a union bound over $v\in N_{1/2}$, yielding 
\begin{align*}
    \mathbb{P}\bigg(\sup_{v\in N_{1/2}}\bigg|\frac{1}{n}\sum_{i=1}^n Y_i(v)\bigg|\lesssim \sqrt{\frac{t}{nR}}+\frac{t}{n}\bigg)\ge 1- 2\exp(d\log(5)-t),\quad \forall t>0.
\end{align*}
Setting $t= d\log (5)+2\log(1/\delta)$ proves the claim.
\end{proof}

\subsubsection{Bounding $T_{c2}$}\label{app:pealing}
We seek to bound
\[T_{c2}:=\bigg\|\frac{1}{n}\sum_{i=1}^n \bigg( \big(s(x_i^\top u)-s(x_i^\top \theta^*)\big)x_i-\mathbb{E}\Big[\big(s(x_i^\top u)-s(x_i^\top \theta^*)\big)x_i\Big]\bigg)\bigg\|_2\]
uniformly for all $u\in \mathbb{B}_2^d(2\|\theta^*\|_2)$. 
Due to the multiplier $s(x_i^\top u)-s(x_i^\top \theta^*)$, one expects a tighter bound for $u$ being close to $\theta^*$. We achieve such ``localization'' via a peeling argument, along with a covering argument over each shell. This establishes the following Lemma \ref{lem:peeling}.

We shall pause to announce the final bound on $T_{c2}$. Setting \[\delta=\exp(-d),\qquad\rho = \frac{d}{n}\]
and applying a relaxation $\log\log_2(\frac{6\|\theta^*\|_2}{\rho})\le \log \frac{Cn\|\theta^*\|_2}{d}$,
we reach the following: if $n\gtrsim d\log n + \log\|\theta^*\|_2$, then with probability at least $1-e^{-d}$,
\begin{align}
    \label{Tc2bound}
    T_{c2} \lesssim \sqrt{\frac{d\log n+ \log\|\theta^*\|_2}{n}}\|u-\theta^*\|_2 + \frac{d}{n},\quad \forall u\in \mathbb{B}_2^d(2\|\theta^*\|_2). 
\end{align}

 \begin{lem}[Bounding $T_{c2}$ via peeling]\label{lem:peeling}
For $u\in\mathbb{R}^d$, define 
\[
Z_n(u)
:=
\frac1n\sum_{i=1}^n
\big(s(x_i^\top u)-s(x_i^\top\theta^*)\big)x_i
-
\mathbb E\big[
\big(s(x^\top u)-s(x^\top\theta^*)\big)x
\big].
\]
Fix $0<\rho\le 3\|\theta^*\|_2$. Then there exists a universal constant $C>0$ such that, with probability at least $1-\delta$, 
\[
\|Z_n(u)\|_2
\le
C\left[
\sqrt{
\frac{
d\log(Cn)+\log\!\left(\frac{\log_2(6\|\theta^*\|_2/\rho)}{\delta}\right)}
{n}}
+
\frac{
d\log(Cn)+\log\!\left(\frac{\log_2(6\|\theta^*\|_2/\rho)}{\delta}\right)}
{n}
\right]
\max\{\|u-\theta^*\|_2,\rho\}
\]
holds uniformly over all $u$ satisfying $\|u\|_2\le 2\|\theta^*\|_2$.
\end{lem}

\begin{proof}
Let $R_0=3\|\theta^*\|_2$ and write $h:=u-\theta^*.$ In light of 
\[\|u-\theta^*\|_2\le \|u\|_2+\|\theta^*\|_2\le 3\|\theta^*\|_2=R_0,\]
we shall restrict our attention to 
\[\|h\|_2\le R_0.\]
For $w\in \mathbb{S}^{d-1}$, define the scalar process
$
Z_n(h,w)
:=
\langle w,Z_n(\theta^*+h)\rangle,$ i.e.,
\[
Z_n(h,w)
=
\frac1n\sum_{i=1}^n
\left[
\big(s(x_i^\top(\theta^*+h))-s(x_i^\top\theta^*)\big)(x_i^\top w)
-
\mathbb E
\big(s(x^\top(\theta^*+h))-s(x^\top\theta^*)\big)(x^\top w)
\right],
\]
and note that
\[\|Z_n(u)\|_2=\sup_{w\in \mathbb{S}^{d-1}}Z_n(u-\theta^*,w).\]

\paragraph{Non-uniform bound.}
We first establish a fixed-$(h,w)$ concentration bound for $Z_n(h,w)$ for $w\in\mathbb{S}^{d-1}$. Since $\sup_{a\in\mathbb{R}}|s'(a)|\le 1/4$,
we have $\left|
s(x_i^\top(\theta^*+h))-s(x_i^\top\theta^*)
\right|
\le
\frac14 |x_i^\top h|,$ and
therefore 
\[
\left|
\big(s(x_i^\top(\theta^*+h))-s(x_i^\top\theta^*)\big)(x_i^\top w)
\right|
\le
\frac14 |x_i^\top h|\,|x_i^\top w|.
\]
Suppose that $\|h\|_2\le 2\tilde{r}$ and $\|a\|_2=1$, then $x_i^\top h$ is sub-Gaussian with $\psi_2$ norm $O(\tilde{r})$. Also, $x_i^\top a$ is sub-Gaussian with $\psi_2$ norm $O(1)$. Hence, the product is sub-exponential with $\psi_1$ norm $O(\tilde{r})$:
\[\|\big(s(x_i^\top(\theta^*+h))-s(x_i^\top\theta^*)\big)(x_i^\top a)\|_{\psi_1}\lesssim \|(x_i^Th)(x_i^Tw)\|_{\psi_1}\le \|x_i^Th\|_{\psi_2}\|x_i^Tw\|_{\psi_2}\lesssim \tilde{r}.\]
See \cite[Section 2]{vershynin2018high} for details. 
By Bernstein's inequality for sub-exponential variables (cf. Lemma \ref{bern}), for every fixed pair $(h,w)$ with $\|h\|_2\le 2\tilde{r}$ and $w\in\mathbb{S}^{d-1}$,
\begin{align}
    & \mathbb P\left(
|Z_n(h,w)|
\gtrsim t
\right)
\le
2\exp\bigg(-cn\min\bigg\{(\frac{t}{\tilde{r}})^2,\frac{t}{\tilde{r}}\bigg\}\bigg),\quad \forall t>0.
    \\
    \Longrightarrow~&\mathbb P\left(
|Z_n(h,w)|
\gtrsim \tilde{r}\left(\sqrt{\frac{t}{n}}+\frac{t}{n}\right)
\right)
\le
2e^{-t},\quad \forall t>0.\label{fixedbound}
\end{align}

\paragraph{Peeling argument: $\|h\|_2\ge \rho$.} We now apply a peeling argument. For the fixed $0<\rho\le R_0:=3\|\theta^*\|_2$ and
\[
j=0,1,\ldots,J,
\qquad
J:=\left\lceil \log_2(R_0/\rho)\right\rceil,
\]
define
\[
r_j:=2^j\rho,
\qquad
\mathcal S_j
:=
\left\{
h:\ r_j\le \|h\|_2\le 2r_j
\right\}.
\]
Let $\mathcal N_j$ be an $(r_j/n)$-net of  
$
\mathcal S_j
=
\left\{
h:\ r_j\le \|h\|_2\le 2r_j
\right\}$, and we may choose
$|\mathcal N_j|\le (Cn)^d$ \cite[Corollary 4.2.13]{vershynin2018high}. Also let $\mathcal M$ be a $1/2$-net of $\mathbb{S}^{d-1}$ with
$
|\mathcal M|\le 5^d$ \cite[Corollary 4.2.13]{vershynin2018high}.
Applying the fixed-$(h,w)$ Bernstein bound in (\ref{fixedbound}) to all pairs
\[
(h_0,w_0)\in \mathcal N_j\times \mathcal M
\]
and taking a union bound over $j=0,\ldots,J$, along with a suitable choice of $t$ to offset the effect of the union bound, we obtain that
\begin{align}\label{discretebound}
    &|Z_n(h_0,w_0)|
\le
C r_j\underbrace{\bigg(\sqrt{
\frac{
d\log(Cn)+\log\!\left(\frac{J+1}{\delta}\right)}
{n}}
+
\frac{
d\log(Cn)+\log\!\left(\frac{J+1}{\delta}\right)}
{n}\bigg)}_{:=A_n},\\&\forall h_0\in \calN_j,~\forall w_0\in\calM,\forall j=0,\cdots,J
\end{align}
holds with probability at least $1-\delta/2$.

It remains to pass from the nets to all points in each shell. For any $h,h'\in\mathbb R^d$ and any $w\in \mathbb{S}^{d-1}$, using again $\sup_{a\in\mathbb{R}}|s'(a)|\le 1/4$, we obtain 
\[
|Z_n(h,w)-Z_n(h',w)|
\le
\frac14
\left[
\frac1n\sum_{i=1}^n
|x_i^\top(h-h')|\,|x_i^\top w|
+
\mathbb E |x^\top(h-h')|\,|x^\top w|
\right].
\]
By Cauchy--Schwarz,
\begin{align*}
\frac1n\sum_{i=1}^n
|x_i^\top(h-h')|\,|x_i^\top w|
&\le
\left(
\frac1n\sum_{i=1}^n (x_i^\top(h-h'))^2
\right)^{1/2}
\left(
\frac1n\sum_{i=1}^n (x_i^\top w)^2
\right)^{1/2} \\
&=
\left((h-h')^\top \widehat\Sigma (h-h')\right)^{1/2}
\left(w^\top \widehat\Sigma w\right)^{1/2},
\end{align*}
where
$
\widehat\Sigma:=\frac1n\sum_{i=1}^n x_i x_i^\top=\frac{1}{n}XX^\top ~~(X=[x_1,...,x_n])$ is the sample covariance. 
Therefore,
\[
\frac1n\sum_{i=1}^n
|x_i^\top(h-h')|\,|x_i^\top w|
\le
\|\widehat\Sigma\|_{\mathrm{op}}\|h-h'\|_2\|w\|_2
=
\|\widehat\Sigma\|_{\mathrm{op}}\|h-h'\|_2=\frac{\|X\|_{op}^2}{n}\|h-h'\|_2.
\]
Similarly, since $x\sim N(0,I_d)$,
\[
\mathbb E |x^\top(h-h')|\,|x^\top w|
\le
\left(\mathbb E(x^\top(h-h'))^2\right)^{1/2}
\left(\mathbb E(x^\top w)^2\right)^{1/2}
=
\|h-h'\|_2\|w\|_2
=
\|h-h'\|_2.
\]
Also, 
\[
\mathbb E |x^\top(h-h')|\,|x^\top w|
\le
\|h-h'\|_2.
\]
Thus
\[
|Z_n(h,w)-Z_n(h',w)|
\le
C\Big(
1+
\frac{\|X\|_{op}^2}{n}
\Big)\|h-h'\|_2.
\]
By Lemma \ref{lem:gaubound}, with probability at least $1-\delta/2$,
\[
\left\|X
\right\|_{\mathrm{op}}
\lesssim 
C\sqrt{n+d + \log(1/\delta)}. 
\]

On the intersection of the these high-probability events, fix any $h\in\mathcal S_j$ and choose $h_0\in\mathcal N_j$ with
$
\|h-h_0\|_2\le \frac{r_j}{n}.$ Then uniformly over $w\in \mathbb{S}^{d-1}$,
\[
|Z_n(h,w)-Z_n(h_0,w)|
\lesssim \bigg(1+\frac{n+d+\log(1/\delta)}{n}\bigg)\frac{r_j}{n}\lesssim  r_j A_n,
\] 
This extends the bound in (\ref{discretebound}) to $\forall h_0\in \calS_j$, $a_0\in \calM$, $j\in0,\cdots,J$, yielding
\begin{align}
    |Z_n(h,w_0)|\le Cr_j A_j,\quad \forall h \in \calS_j,~\forall w_0\in\calM,~\forall j =0,\cdots,J,\label{nownowbound}
\end{align}
so long as the constant $C$ is increased accordingly.

It remains to show that $w_0\in\calM$ is sufficient for bounding 
\[\|Z_n(\theta^*+h)\|_2=\sup_{w\in \mathbb{S}^{d-1}}Z_n(h,w).\]
 Indeed, a standard discretization argument (see, e.g., \cite[Exercise 4.4.2]{vershynin2018high}) yields \[
\|Z_n(\theta^*+h)\|_2\le 2\sup_{w_0\in\mathcal M}|\langle w_0,Z_n(\theta^*+h)\rangle|=2\sup_{w_0\in\calM}|Z_n(h,w_0)|.  
\] 
Therefore, by the bound in (\ref{nownowbound}), for every $h\in\mathcal S_j$,
\[
\|Z_n(\theta^*+h)\|_2
\le
C r_j A_n.
\]
Since $\|h\|_2\ge r_j$ on $\mathcal S_j$ by definition, we get
\[
\|Z_n(\theta^*+h)\|_2
\le
C A_n\|h\|_2,\quad \forall \rho\le \|h\|_2\le R_0
\]

\paragraph{Handling $\|h\|_2\le \rho$.} It remains to handle the small ball $\|h\|_2\le \rho$. The same argument as above, applied with $\tilde{r}=\rho$ in Equation (\ref{fixedbound}) and the $(\rho/n)$-net of the ball
$
\{h:\|h\|_2\le \rho\},$ gives 
\[
\|Z_n(\theta^*+h)\|_2
\lesssim  A_n\rho,\qquad \forall h\textrm{ ~obeying~~}\|h\|_2\le \rho. 
\]
We omit further details for this step. 

Combining the shell bound and the small-ball bound, we obtain uniformly over all
$\|u\|_2\le 2\|\theta^*\|_2$,
\[
\|Z_n(u)\|_2
\le
C A_n\max\{\|u-\theta^*\|_2,\rho\}.
\]
Since $J+1\le C\log_2(2R_0/\rho)$, this proves the claim.
\end{proof}

\subsubsection{Final bound on $T_c$}\label{sec:finalTc}
We now substitute the bounds in Equations (\ref{Tc1bound}) and (\ref{Tc2bound}) into (\ref{decompconcen}) to obtain our final bound on the concentration term $T_c$: if $n\gtrsim d\log n+\log\|\theta^*\|_2$, then with probability at least $1-2e^{-d}$, we have
\begin{align}\label{finalTc}
    \|h_{\theta^*}(u)-\mathbb{E}[h_{\theta^*}(u)]\|_2 \lesssim \sqrt{\frac{d\log n+ \log\|\theta^*\|_2}{n}}\|u-\theta^*\|_2 + \sqrt{\frac{d}{n\|\theta^*\|_2}}+ \frac{d}{n},~~ \forall u\in \mathbb{B}_2^d(2\|\theta^*\|_2). 
\end{align}

\subsection{Bounding $T_e$ (bias term)} \label{app:biasthm1}
We now control the  bias term 
\begin{align*}
    T_e := \big\|u-\theta^*-\eta\cdot\mathbb{E}[h_{\theta^*}(u)]\big\|_2,
\end{align*}
where by $h_{\theta^*}(u)=\frac{1}{n}\sum_{i=1}^n (s(x_i^\top u)-y_i)x_i$ and $\mathbb{E}[y_i|x_i]=s(x_i^\top \theta^*)$
 \begin{align*}
     \mathbb{E}[h_{\theta^*}(u)] = \mathbb{E}\big[\big(s(x_i^\top u)-s(x_i^\top \theta^*)\big)x_i\big].
 \end{align*}
 Now an important step is to apply stein's identity, 
 \begin{align*}
     \mathbb{E}[h_{\theta^*}(u)] &= \mathbb{E}[x_is(x_i^\top u)]-\mathbb{E}[x_is(x_i^\top \theta^*)]\\& = \mathbb{E}[s'(x_i^\top u)]u - \mathbb{E}[s'(x_i^\top \theta^*)]\theta^*\\
     \explain \textrm{by Equation (\ref{stein3}) in Lemma \ref{lem:stein}}\\
     & = F(u)-F(\theta^*).\\
     \explain\textrm{we introduce $F(u):=\mathbb{E}[s'(x_i^\top u)]u=m(\|u\|_2)u$}
 \end{align*}
Recall that $m(\tau):=\mathbb{E}_{g\sim N(0,1)}s'(\tau g)$. The main work of the analysis on this bias term is to identify $\eta$ such that 
\begin{align}\label{T2expression}
    T_e:= \big\|u-\theta^*-\eta\cdot\mathbb{E}[h_{\theta^*}(u)]\big\|_2 = \big\|u-\theta^*-\eta(F(u)-F(\theta^*))\big\|_2
\end{align}
is a contraction. Under the small stepsize $\eta\in(0,8)$, we establish the following result by the main idea of approximating $F(u)-F(\theta^*)$ via a linear map on $u-\theta^*$.

We shall pause to see the implication on $T_e$. By (\ref{T2expression}), we shall set $v=\theta^*$ in Lemma \ref{lem:smallcontract} to establish that, under $\eta\in[0.1,7.9]$ (as chosen in Theorem \ref{thm:small}), 
\begin{align}\label{T2bound}
    T_e\le \bigg(1-\frac{c}{
    \|\theta^*\|_2^3
    }\bigg) \|u-\theta^*\|_2,\quad \forall u\in\mathbb{B}_2^d(2\|\theta^*\|_2). 
\end{align}
holds for some universal constant $c>0$. 

\begin{lem}[$T_2$ is a contraction under small stepsize]\label{lem:smallcontract}
Let $F(z)=m(\|z\|_2)z$ and $m(\tau)=\mathbb{E}_{g\sim N(0,1)}[s'(\tau g)]$.
Fix $0<\eta<8$. Let $v\in\mathbb{R}^d$ satisfy $\|v\|_2=R\ge 1$. 
For any $u\in\mathbb{R}^d$ satisfying 
$
\|u\|_2\le 2R,$
we have that 
\[
\|\eta(F(u)-F(v))-(u-v)\|_2
\le
\left(1-\frac{c\eta\left(2-\frac{\eta}{4}\right)}{R^3}\right)\|u-v\|_2.
\]
\end{lem}

\begin{proof} 
Let $h:=u-v$, then by Lemma \ref{lem:integral}, we have
\[
F(u)-F(v)=\int_0^1 \nabla F(v+th)h\,dt
\]
where \begin{align}
    \nabla F(z)
=
m(\|z\|_2)I
+
\frac{m'(\|z\|_2)}{\|z\|_2}zz^\top,~~ z\ne 0\quad \textrm{ with the convention $\nabla F(0)= m(0)I_d$.}\label{nablaFz}
\end{align}
Define the matrix
\begin{align}
    A:=\int_0^1 \nabla F(v+th)\,dt,\label{defineA}
\end{align}
then
\begin{align}
    \label{FuFvA}
    F(u)-F(v)=Ah. 
\end{align}
\paragraph{Spectrum of $A.$}
For every $t\in[0,1]$,
since $\|u\|_2\le 2R$ and $\|v\|_2=R$, and $v+th=(1-t)v+tu$, we have 
\[
\|v+th\|_2
\le (1-t)\|v\|_2+t\|u\|_2
\le 2R.
\]
In light of (\ref{nablaFz}), if $z\ne 0$, then the eigenvalues of $\nabla F(z)$ are
$m(\|z\|_2)$ in directions orthogonal to $z$, and 
\begin{align}
    m(\|z\|_2)+\|z\|_2m'(\|z\|_2)
=q'(\|z\|_2)
\label{introduceqt}
\end{align}
in the   direction of $\frac{z}{\|z\|_2}$; we mention that this extends to $z=0$. Note that in (\ref{introduceqt}), we introduce 
\[q(\tau)=\tau m(\tau)\]
as in Lemma \ref{lem:separa}, so that 
\begin{align}
    q'(\|z\|_2)= m(\|z\|_2)+\|z\|_2 m'(\|z\|_2)=\mathbb{E}[s'(\|z\|_2g)g^2]\le \frac{1}{4}\mathbb{E}[g^2]=\frac{1}{4}.\label{qprimeupper}
\end{align}
By Lemma \ref{lem:separa}, we have  $m(\tau)\asymp \frac{1}{1+\tau},
~q'(\tau)\asymp \frac{1}{(1+\tau)^3}$ for $\tau>0$. Combining with $\|v+th\|_2\le 2R$ and $R\ge 1$, we obtain that all the eigenvalues of $\nabla F(v+th)$, for any $t\in[0,1]$, are of order $\Omega(R^{-3})$. This means that 
\begin{align}
    \nabla F(v+th)\succeq  \frac{c}{R^3}I\quad\textrm{for some absolute constant $c$.} \label{pdnablaF}
\end{align}and therefore for any $w\in \mathbb{R}^d$, 
\begin{align}\label{lambdaminA}
    \langle w,Aw\rangle &= \int_0^1  w^\top \nabla F(v+th)w\,dt\ge \frac{c\|w\|_2^2}{R^3}\\\explain\textrm{by Equations (\ref{defineA}) and (\ref{pdnablaF})} 
\end{align}
On the other hand, since $s'(a)\le 1/4$, in light of (\ref{introduceqt}) and (\ref{qprimeupper}) we have
\begin{align*}
    \|\nabla F(v+th)\|_{op}\le m(\|v+th\|_2)\le \frac{1}{4}
\end{align*}
which then yields \begin{align}
    \langle w,Aw\rangle=\int_0^1w^\top \nabla F(v+th) w\,dt \le\frac{1}{4}\|w\|_2^2,\quad \forall w\in\mathbb{R}^d.\label{lammaxbound}
\end{align} 
Combining (\ref{lambdaminA}) and (\ref{lammaxbound}), we conclude that the positive definite  matrix $A$ satisfies
\begin{align*}
      \frac{c}{R^3}I_d\preceq  A \preceq \frac{1}{4}I_d,  
\end{align*}
and we can eigenvalue decomposition of $A$ as 
\begin{align}\label{eigA}
    A=O^\top \Lambda O 
\end{align}
for some orthonormal matrix $O$ and a diagonal matrix 
\[\Lambda=\diag(\lambda_1,\cdots,\lambda_d),\quad \lambda_i \in \Big[\frac{c_1}{R^3} ,\frac{1}{4}\Big],~i\in[d].\]
Now using (\ref{FuFvA}) and $h=u-v$, 
\[
\eta(F(u)-F(v))-(u-v)
=
(\eta A-I)h.
\]
Hence
\begin{align}\label{desiredsquare}
\|\eta(F(u)-F(v))-(u-v) \|_2^2&=\|(\eta A-I)h\|_2^2
=
\|h\|_2^2
-2\eta\langle h,Ah\rangle
+\eta^2\langle h,A^2h\rangle. 
\end{align}
To   bound $\langle h,A^2h\rangle$, we shall use the eigenvalue decomposition of $A$ in (\ref{eigA}) to obtain 
\begin{align}\label{upperhAAh}
    h^\top \Big(A^2-\frac{1}{4}A\Big)h = (Oh)^\top \Big(\Lambda^2-\frac{\Lambda}{4}\Big)(Oh)\le 0 ~~\Longrightarrow~~ h^\top A^2h\le \frac{1}{4}h^\top Ah. 
\end{align}
To lower bound $\langle h,Ah\rangle$, Equation (\ref{lambdaminA}) yields
\begin{align}\label{lowerhAh}
    \langle h,Ah\rangle \ge \frac{c\|h\|_2^2}{R^3}.
\end{align}
Substituting  (\ref{upperhAAh}) and (\ref{lowerhAh}) into (\ref{desiredsquare}), we obtain
\begin{align*}
    \|\eta(F(u)-F(v))-(u-v) \|_2^2 \le \|h\|_2^2-(2\eta-\frac{\eta^2}{4}) h^\top Ah \le \bigg(1-\frac{c\eta(2-\eta/4)}{R^3}\bigg)\|h\|_2^2
\end{align*}
Taking square roots and using $\sqrt{1-a}\le 1-\frac{a}{2}$ for $a\in[0,1)$, we arrive at 
\[
  \|\eta(F(u)-F(v))-(u-v) \|_2
\le
\sqrt{1-\frac{c\eta(2-\eta/4)}{R^3}}\|u-v\|_2\le \bigg(1-\frac{c\eta(2-\eta /4)}{2R^3}\bigg)\|u-v\|_2.
\]
The result follows.  
\end{proof}

 \subsection{AIC and   Convergence}\label{app:aic2con}
 
We are now ready to put pieces together to establish our main structure condition, the approximate invertibility condition (AIC). Under $m\gtrsim d\log n +\log\|\theta^*\|_2$, the bound (\ref{finalTc}) on $T_c$ holds with probability at least $1-2e^{-d}$, and also under $\eta\in[0.1,7.9]$ the bound (\ref{T2bound}) on $T_e $ holds deterministically. We now substitute these two bounds into (\ref{aicdecompose1}) to obtain 
\begin{align}\nn
    \|u-\theta^*-\eta\cdot h_{\theta^*}(u)\|_2 &\le \bigg(1-\frac{c}{\|\theta^*\|_2^3}+C\sqrt{\frac{d\log n+\log\|\theta^*\|_2}{n}}\bigg) \|u-\theta^*\|_2\\
\label{aicsmall}    &+ C'\bigg(\sqrt{\frac{d}{n\|\theta^*\|_2}}+\frac{d}{n}\bigg),\qquad \forall u\in \mathbb{B}_2^d(2\|\theta^*\|_2).
\end{align}
Following recent works \cite{matsumoto2024binary,matsumoto2025learning,chen2025unified,friedlander2021nbiht,abdalla2026robust}, (\ref{aicsmall}) is referred to as an  \emph{approximate invertibility condition} (AIC) in the sense that it essentially controls the difference between the ideal descent step $u-\theta^*$ and the actual gradient step $\eta\cdot h_{\theta^*}(u)$. To imply the desired convergence, we need the linear term, $(1-\frac{c}{\|\theta^*\|_2^3}+\sqrt{\frac{d\log n+\log\|\theta^*\|_2}{n}})\cdot\|u-\theta^*\|_2$, to be a contraction. To that end we enforce the sample complexity
\begin{align}
    \label{samcom}
    n\gtrsim \|\theta^*\|_2^6 d\log n + \|\theta^*\|_2^6\log\|\theta^*\|_2,
\end{align}
under which the AIC (\ref{aicsmall}) simplifies to 
\begin{align}\tag{AIC1}
    \|u-\theta^*-\eta\cdot h_{\theta^*}(u)\|_2 \le \bigg(1-\frac{c'}{\|\theta^*\|_2^3}\bigg)\|u-\theta^*\|_2+ C'\bigg(\sqrt{\frac{d}{n\|\theta^*\|_2}}+\frac{d}{n}\bigg),~~
    \forall u\in \mathbb{B}_2^d(2\|\theta^*\|_2). \label{smallraic2}
\end{align}

To complete the proof of Theorem \ref{thm:small}, all that remains is to show that the above (\ref{smallraic2}) implies the claimed linear convergence. We address this via the following lemma by an argument similar to \cite[Theorem 3.1]{chen2025unified}, which is to control $\{\|\theta_t-\theta^*\|_2\}_{t=0}^\infty$ by a sequence defined iteratively. 

\begin{lem}[(\ref{smallraic2}) implies convergence] \label{lem:aic1toconverge}
    Under (\ref{samcom}) and (\ref{smallraic2}), $\{\theta_t\}_{t\ge 0}$ generated by Algorithm  \ref{alg:gd} with $\theta_0=0$ (and the $\eta $  in (\ref{smallraic2})) satisfies
    \begin{align*}
        \|\theta_t - \theta^*\|_2 \le \bigg(1-\frac{c'}{\|\theta^*\|_2^3}\bigg)^t\|\theta^*\|_2 +  \tilde{C}\sqrt{\frac{\|\theta^*\|_2^5d}{n}},\qquad \forall t=0,1,\cdots . 
    \end{align*}
\end{lem}
\begin{proof}
    We define a sequence by $\{f_t\}_{t=0}^\infty$ by the recurrence and the initial value
    \[f_{t+1}=\bigg(1-\frac{c'}{\|\theta^*\|_2^3}\bigg)f_t+C'\bigg(\sqrt{\frac{d}{n\|\theta^*\|_2}}+\frac{d}{n}\bigg),~t\ge 0,\quad\textrm{and}\quad f_0=\|\theta^*\|_2.\]
    It is not hard to observe the formula 
    \begin{align}
        \label{closedft}f_t= \bigg(1-\frac{c'}{\|\theta^*\|_2^3}\bigg)^t\|\theta^*\|_2+\bigg[1-\Big(1-\frac{c'}{\|\theta^*\|_2^3}\Big)^t\bigg]\frac{C'}{c'}\bigg(\sqrt{\frac{d\|\theta^*\|_2^5}{n}}+\frac{d\|\theta^*\|_2^3}{n}\bigg),\quad\forall t\ge 0.
    \end{align}    
    Hence, under (\ref{samcom}) and $\|\theta^*\|_2\ge 1$, we have that
    \begin{align*}
        f_t\le  \bigg(1-\frac{c'}{\|\theta^*\|_2^3}\bigg)^t\|\theta^*\|_2+ O\bigg(\sqrt{\frac{\|\theta^*\|_2^5d}{n}}\bigg),\quad \forall t\ge 0,
    \end{align*}
    and in turn we only need to prove $\|\theta_t-\theta^*\|_2\le f_t$ for any $t\ge 0.$ Before that, we further observe that  (\ref{samcom}) and $\|\theta^*\|_2\ge 1$ imply \[\frac{C'}{c'}\bigg(\sqrt{\frac{d\|\theta^*\|_2^5}{n}}+\frac{d\|\theta^*\|_2^3}{n}\bigg)\le \|\theta^*\|_2,\]
    hence (\ref{closedft}) also yields $f_t\le \|\theta^*\|_2$ for any $t\ge 0.$

    It remains to prove $\|\theta_t-\theta^*\|_2\le f_t$ for any $t\ge 0$. We accomplish this by induction. Since $\theta_0=0$ and $f_0=\|\theta^*\|_2$, $\|\theta_0-\theta^*\|_2\le f_0$ holds trivially. Now suppose that $\|\theta_t-\theta^*\|_2\le f_t$, and we seek to prove $\|\theta_{t+1}-\theta^*\|_2\le f_{t+1}$.  First by $f_t\le \|\theta^*\|_2$, the hypothesis yields $\|\theta_t\|_2\le f_t+\|\theta^*\|_2\le 2\|\theta^*\|_2$. Therefore we can apply (\ref{smallraic2}) with $u=\theta_t$, yielding 
    \[ \|\theta_t-\theta^*-\eta\cdot h_{\theta^*}(\theta_t)\|_2 \le \bigg(1-\frac{c'}{\|\theta^*\|_2^3}\bigg)\|\theta_t-\theta^*\|_2+ C'\bigg(\sqrt{\frac{d}{n\|\theta^*\|_2}}+\frac{d}{n}\bigg).\]
    Since $\theta_{t+1}=\theta_t-\eta \cdot h_{\theta^*}(\theta_t)$, and also using the hypothesis $\|\theta_t-\theta^*\|_2\le f_t$, the above display yields 
    \[\|\theta_{t+1}-\theta^*\|_2\le\bigg(1-\frac{c'}{\|\theta^*\|_2^3}\bigg)f_t+ C'\bigg(\sqrt{\frac{d}{n\|\theta^*\|_2}}+\frac{d}{n}\bigg) =f_{t+1}.\]
    This completes the induction. 
\end{proof}

\section{Proof of Theorem \ref{thm:big} (GD with Large Stepsize)}\label{app:proofbig}
The overall proof architecture remains the same as Theorem \ref{thm:small}, that is to establish an AIC and then use it to imply the desired convergence. The first main difference is that the AIC for Theorem \ref{thm:big} only holds ``locally'' over   a small radius-$O(\|\theta^*\|_2^{-1})$ neighborhood of $\theta^*$. Again we start with (\ref{aicdecompose1}), i.e., the decomposition
 \begin{align}\label{aicdecompose2}
     \|u-\theta^*-\eta\cdot h_{\theta^*}(u)\|_2\le \eta\cdot \underbrace{\|h_{\theta^*}(u) - \mathbb{E}[h_{\theta^*}(u)]\|_2}_{:=T_c} + \underbrace{\|u-\theta^* -\eta\mathbb{E}[h_{\theta^*}(u)]\|_2}_{:=T_e},
 \end{align}
and we shall restrict to $u\in \mathbb{B}_2^d(\theta^*;\frac{c_0}{\|\theta^*\|_2})$ for some small enough constant $c_0$. The major technical difference arises in the analysis of the bias term $T_e$. 

 \subsection{Bounding $T_c$ (concentration term)}
While re-iterating the arguments in Appendix \ref{sec:boundTc} (especially Appendix \ref{app:pealing}) over a smaller range of $u$, namely $\mathbb{B}_2^d(\theta^*;\frac{c_0}{\|\theta^*\|_2})$, may lead to slightly tighter bound, such improvement is marginal and can only refine log factors. Therefore, we directly reuse the previous bound, which we restate here (cf. Appendix \ref{sec:finalTc}): if $n\gtrsim d\log n+\log\|\theta^*\|_2$, then with probability at least $1-2e^{-d}$, we have
\begin{align}\label{finalTc1}
    \|h_{\theta^*}(u)-\mathbb{E}[h_{\theta^*}(u)]\|_2 \lesssim \sqrt{\frac{d\log n+ \log\|\theta^*\|_2}{n}}\|u-\theta^*\|_2 + \sqrt{\frac{d}{n\|\theta^*\|_2}}+ \frac{d}{n},~~ \forall u\in \mathbb{B}_2^d(\theta^*;\frac{c_0}{\|\theta^*\|_2}). 
\end{align}

 \subsection{Bounding $T_e$ (bias term)}\label{app:bias_thm2}
We shall now bound $T_e$ over $u\in \mathbb{B}_2^d(\theta^*;\frac{c_0}{\|\theta^*\|_2})$. The main ingredient is the following Lemma \ref{lem:largeeta}, which involves two functions that are recurring in our analysis: \[\textrm{$m(\tau)= \mathbb{E}[s'(\tau g)]$ and $q'(\tau)=\mathbb{E}[s'(\tau g)g^2]$}\] 
See,e.g., Theorem \ref{thm:big} and Lemma \ref{lem:separa}. 
An important starting point is, again,  the transform in   (\ref{T2expression}) based on stein's identity: 
\begin{align}
    T_e:= \big\|u-\theta^*-\eta\cdot\mathbb{E}[h_{\theta^*}(u)]\big\|_2 = \big\|u-\theta^*-\eta(F(u)-F(\theta^*))\big\|_2,\label{Tetransform}
\end{align}
where $F:\mathbb{R}^d\to\mathbb{R}^d$ is defined as $F(u)=m(\|u\|_2)u$.

Let us pause to see the implied bound of Lemma \ref{lem:largeeta} on $T_e$. We set $v=\theta^*$ in Lemma \ref{lem:largeeta}, then (\ref{largeeta1}) is also the stepsize assumption in Theorem \ref{thm:big}, and the conclusion 
\[\left\|
\eta(F(u)-F(\theta^*))-(u-\theta^*)
\right\|_2
\le
\left(1-\frac{c_1}{R^2}\right)\|u-\theta^*\|_2,\quad \forall u\in \mathbb{B}_2^d(\theta^*;\frac{c_0}{\|\theta^*\|_2}) ,\]
in light of (\ref{Tetransform}), yields 
\begin{align}
    T_e\le \left(1-\frac{c_1}{\|\theta^*\|_2^2}\right)\|u-\theta^*\|_2,\quad \forall u\in \mathbb{B}_2^d(\theta^*;\frac{c_0}{\|\theta^*\|_2}). \label{finalTeboundlarge}
\end{align}
This is final bound on the bias term $T_e.$
\begin{lem} \label{lem:largeeta} 
There exist universal constants $c,c_0,c_1>0$ such that the following holds. 
Let $v\in\mathbb{R}^d$ satisfy
$R:=\|v\|_2\ge 1.$ If 
\begin{align}
    \label{largeeta1}
     \frac{c}{R^2q'(R)}<\eta<\frac{2-cR^{-2}}{m(R)},
\end{align}
  then for any $\|u-v\|_2\le \frac{c_0}{R}$ we have 
\[
\left\|
\eta(F(u)-F(v))-(u-v)
\right\|_2
\le
\left(1-\frac{c_1}{R^2}\right)\|u-v\|_2.
\]
Moreover, $\eta= \frac{1}{m(R)}$ is a specific choice satisfying (\ref{largeeta1}). 
\end{lem}

\begin{proof}
Let $h:=u-v.$ If $h=0$, the result is trivial, so we assume $h\ne 0$. Let
$
e:=\frac{v}{\|v\|_2},$ and we shall decompose $h$ along $e$ and an orthogonal direction  $\frac{w}{\|w\|_2}$ (here, $\{e,\frac{w}{\|w\|_2}\}$ is orthonormal basis for ${\rm span}(u,v)$, and in the case of $u = \lambda v$ for some $\lambda\in\mathbb{R}$ we can choose $w/\|w\|_2$ as an arbitrary unit vector orthogonal to $v$):
\begin{align}\label{decomposeh}
    h=\alpha e+w,\qquad \textrm{for some }w\perp e. 
\end{align}
With these conventions and $R=\|v\|_2\ge 1$ we  have
\begin{align}
    u = v+h =(R+\alpha)e+w,
\label{uexpress}
\end{align}
and hence $\rho:=\|u\|_2$ equals 
\begin{align}
    \rho:=\|u\|_2=\sqrt{(R+\alpha)^2+\|w\|_2^2}.\label{rhoequa}
\end{align}
Since $\|h\|_2\le \frac{c_0}{R}$ and $R\ge 1$, choosing sufficiently small $c_0$ (together with triangle inequality) ensures 
\begin{align}\label{rhioatR}
    \frac{R}{2}\le R-\|h\|_2\le\rho \le R+\|h\|_2 \le 2R
\end{align} 
The above conventions, along with $q(\tau)=\tau m(\tau)$ (cf. Lemma \ref{lem:separa}), also allow us to write 
\begin{align*}
    F(v)=m(\|v\|_2)v=m(R)Re=q(R)e, 
\end{align*}
and in light of (\ref{uexpress})
\begin{align*}
F(u)=m(\|u\|_2)u=m(\rho)\big((R+\alpha)e+w\big).
\end{align*}
Hence
\begin{align}\label{FuFv}
    F(u)-F(v)
=
\Big(m(\rho)(R+\alpha)-q(R)\Big)e
+
m(\rho)w. 
\end{align}
\paragraph{Approximating $F(u)-F(v)$ via linear map.} We now come to the most technical step in this proof, which approximates $F(u)-F(v)$  via a linear map on $h$.  We introduce 
\begin{align}
    J_v=q'(R)ee^\top+m(R)(I-ee^\top).\label{defineJv}
\end{align}
Then for $h=\alpha e+w$ with $w\perp e$ (cf. Equation (\ref{decomposeh})), we have 
\[
ee^\top h=\alpha e,\qquad (I-ee^\top)h=w,
\]
and hence
\begin{align}
    J_vh=q'(R)ee^\top h+m(R)(I-ee^\top)h
=q'(R)\alpha e+m(R)w.\label{Jvh}
\end{align}
We want to use $J_vh$ to approximate $F(u)-F(v)=F(v+h)-F(v)$, with the intuition behind being that $J_v$ in (\ref{defineJv}) is just the Jacobian matrix $\nabla F(v)$ (see Lemma \ref{lem:integral}). 
Our main claim is that
\begin{align}
    F(u)-F(v)=J_vh+\mathcal R,\quad \textrm{with ~$\|\mathcal R\|_2\le C\frac{\|h\|_2^2}{R^2}.$}\label{approximate}
\end{align} 
\paragraph{Proof of (\ref{approximate}).} In light of (\ref{FuFv}) and (\ref{Jvh}),  
\begin{align}\label{seset}
    F(u)-F(v)-J_vh = \underbrace{\Big(m(\rho)(R+\alpha)-q(R)- q'(R)\alpha\Big)}_{:=T_1}e + \underbrace{\big(m(\rho)-m(R)\big)}_{T_2}w
\end{align}
and we shall treat the components of $e$ and of $w$ separately.

For the $e$-component, by $q(\rho)=\rho m(\rho)$ we have 
\begin{align*}
T_1 =
\underbrace{q(\rho)-q(R)-q'(R)(\rho-R)}_{:=T_1'
}
+
\underbrace{\big(q'(R)-m(\rho)\big)(\rho-R-\alpha)}_{T''_1}.
\end{align*}
To bound $T_1'$, we have
\begin{align}\nn
    &|T'_1|=\big|q(\rho)-q(R)-q'(R)(\rho-R)\big|\\
\nn    &\lesssim \Big(\sup_{\tau\ge\frac{R}{2}}|q''(\tau)|\Big)\cdot|\rho-R|^2\\
    \explain\textrm{by Taylor's theorem and Equation (\ref{rhioatR})}\\
    &\lesssim \frac{\|h\|_2^2}{R^4}\label{firstT1esti}\\\explain\textrm{by $q''(\tau)\asymp \frac{1}{\tau^4}$ and $|\rho-R|\le \|h\|_2$}
\end{align}
We now bound $|T''_1|=|q'(R)-m(\rho)||\rho-R-\alpha|$. 
The   bounds in Lemma \ref{lem:separa} give 
\[
|q'(R)-m(\rho)|
\le |q'(R)|+|m(\rho)|
\lesssim \frac1R.
\]
Combining with \[
|\rho-R-\alpha| 
=
\frac{\rho^2-(R+\alpha)^2}{\rho+R+\alpha}
\stackrel{(\ref{rhoequa})}{=}\frac{\|w\|_2^2}{\rho+R+\alpha}\stackrel{(\ref{decomposeh})}{\lesssim}
\frac{\|w\|_2^2}{R}
\stackrel{(\ref{decomposeh})}{\le}
\frac{\|h\|_2^2}{R}, 
\] 
we obtain 
\begin{align}
    |T_1''|
\lesssim
\frac1R\cdot \frac{\|h\|_2^2}{R}
=
\frac{\|h\|_2^2}{R^2}.\label{secondT1esti}
\end{align}
Combining (\ref{firstT1esti}) and (\ref{secondT1esti}) yields 
\begin{align}
    |T_1|=O\left(\frac{\|h\|_2^2}{R^2}\right).
\label{T1bound}
\end{align}
For $T_2$, by
using $|m'(\tau)|\lesssim \tau^{-2}$ for $\tau\ge\frac{1}{2}$ (cf. Lemma \ref{mporder}), we obtain 
\[|T_2|=
|m(\rho)-m(R)|\stackrel{(\ref{rhioatR})}{\le} \sup_{\tau\ge R/2}|m'(\tau)|\cdot |\rho-R|
\lesssim
\frac{|\rho-R|}{R^2}
\stackrel{(\ref{rhioatR})}{\le}
\frac{\|h\|_2}{R^2}.
\]
Together with $\|w\|_2\le \|h\|_2$ from (\ref{decomposeh}), we reach 
\begin{align}
    \|T_2w\|_2
\lesssim
\frac{\|h\|_2}{R^2}\|w\|_2
\le
\frac{\|h\|_2^2}{R^2}.\label{T2bound1}
\end{align}
We are now ready to establish the claimed approximation guarantee in (\ref{approximate}):
\begin{align*}
    \big\|F(u)-F(v)-J_vh\big\|_2 &\stackrel{(\ref{seset})}{\le} |T_1|+|T_2|\|w\|_2\\
    &~\le \frac{\|h\|_2^2}{R^2}.\\
    \explain\textrm{by Equations (\ref{T1bound}) and (\ref{T2bound1})}
\end{align*}

\paragraph{Concluding the proof.} In view of (\ref{defineJv}), $J_v$ has eigenvalue
$
\lambda_\parallel:=q'(R)
$
in the direction $e$ and eigenvalue
$
\lambda_\perp=m(R)
$
on the orthogonal subspace. Hence Lemma \ref{lem:V1larger} yields $\lambda_{\|}<\lambda_{\perp}$. 
Note that our choice of stepsize in (\ref{largeeta1}) 
is equivalent to 
\begin{align*}
    -\bigg(1-\frac{c}{R^2}\bigg)<1-\eta\lambda_{\perp}<1-\eta\lambda_{\parallel}<1-\frac{c}{R^2},
\end{align*}
which then also ensures
\begin{align}
\label{contract}
\|\eta J_v-I_d\|_{op} \le 1-\frac{c}{R^2}.
\end{align}
Using (\ref{approximate}), 
we get
\begin{align*}
&\|\eta(F(u)-F(v))-(u-v)\|_2
\\& = \|\eta J_vh+\calR-h\|_2
\\&\le
\|(\eta J_v-I_d)h\|_2+\eta\|\mathcal R\|_2\\
&\le
\left(1-\frac{c}{R^2}\right)\|h\|_2
+
\eta C\frac{\|h\|_2^2}{R^2}\\
&\le \left(1-\frac{c}{R^2}\right)\|h\|_2+C'\frac{\|h\|_2^2}{R}\\\explain\textrm{By $m(R)\asymp  \frac{1}{R}$ from Lemma \ref{lem:separa}, we have $\eta \lesssim R$}\\
&\le \left(1-\frac{c}{R^2}\right)\|h\|_2+ \frac{C'c_0\|h\|_2}{R^2} \\\explain\textrm{by $\|h\|_2\le \frac{c_0}{R}$ as assumed}\\
&\le \left(1-\frac{c}{2R^2}\right)\|h\|_2.\\
\explain\textrm{Taking $c_0$ sufficiently small}
\end{align*}
 This proves the claim. 
 
\paragraph{Specific stepsize $\eta=\frac{1}{m(R)}$.} Finally, we prove the ``moreover'' part of the statement. We claim that the simpler choice
$
\eta=\frac{1}{m(R)}
$
also belongs to the feasible interval
$
\left[
\frac{c}{R^2q'(R)},\,
\frac{2-cR^{-2}}{m(R)}
\right]
$ 
provided $c>0$ is sufficiently small. 
Indeed, the upper bound is immediate. Since $R\ge 1$, if $c\le 1$, then
$
1\le 2-cR^{-2},
$
and hence
$
\frac{1}{m(R)}
\le
\frac{2-cR^{-2}}{m(R)}.
$
For the lower bound, we need
$
\frac{c}{R^2q'(R)}
\le
\frac{1}{m(R)}.
$ 
Since $m(R),q'(R)>0$, this is equivalent to
$
c\,m(R)\le R^2q'(R).$ 
Using the   estimates from Lemma \ref{lem:separa}, 
we have
\[
\frac{R^2q'(R)}{m(R)}
\gtrsim
\frac{R^2 R^{-3}}{R^{-1}}
\gtrsim 1~~\Longrightarrow~~
\inf_{R\ge 1}\frac{R^2q'(R)}{m(R)}>0. 
\]
Choosing $c>0$ smaller than this infimum and also $c\le 1$ yields
$
\frac{1}{m(R)}
\in
\left[
\frac{c}{R^2q'(R)},\,
\frac{2-cR^{-2}}{m(R)}
\right]
$ for all $R\ge 1$. 
\end{proof}

\subsection{AIC and Convergence}\label{app:localaictov}
We are in a position to establish the AIC.  Substituting (\ref{finalTc1}) and (\ref{Tetransform}) into (\ref{aicdecompose2}), along with $\eta\le \frac{2}{m(\|\theta^*\|_2)}\asymp \|\theta^*\|_2$ (cf. Lemma \ref{lem:separa}), it follows that
\begin{align}\nn
    &\|u-\theta^*-\eta\cdot h_{\theta^*}(u)\|_2 \le C'\|\theta^*\|_2 \sqrt{\frac{d\log n +\log\|\theta^*\|_2}{n}}\|u-\theta^*\|_2 \\
\label{aic2initial} &\quad + \bigg(1-\frac{c_1}{\|\theta^*\|_2^2}\bigg)\|u-\theta^*\|_2 + C''\bigg(\sqrt{\frac{\|\theta^*\|_2d}{n}}+\frac{\|\theta^*\|_2d}{n}\bigg),\quad \forall u\in\mathbb{B}_2^d(\theta^*;\frac{c_0}{\|\theta^*\|_2}).
\end{align}
To ensure that the pre-factor of $\|u-\theta^*\|_2$ is bounded away from $1$, i.e., 
\[1-\frac{c_1}{\|\theta^*\|_2^2}+C'\|\theta^*\|_2 \sqrt{\frac{d\log n+\log\|\theta^*\|_2}{n}}<1-\frac{c_1}{2\|\theta^*\|_2^2},\]
we enforce the sample complexity 
\begin{align}\label{samcom2}
    n\gtrsim \|\theta^*\|_2^6d\log n +\|\theta^*\|_2^6\log\|\theta^*\|_2,
\end{align}
which is (surprisingly) identical to (\ref{samcom}). Then, the AIC in (\ref{aic2initial}) leads to the following local AIC: 
\begin{align}\tag{AIC2}
    \label{localaic}
    \|u-\theta^*-\eta\cdot h_{\theta^*}(u)\|_2 \le \bigg(1-\frac{c'}{\|\theta^*\|_2^2}\bigg)\|u-\theta^*\|_2 +C'\bigg(\sqrt{\frac{\|\theta^*\|_2d}{n}}+\frac{\|\theta^*\|_2d}{n}\bigg),~~ \forall  u\in\mathbb{B}_2^d(\theta^*;\frac{c_0}{\|\theta^*\|_2})
\end{align}
for some universal constants $c',C',c_0$ (after updating the value of $C'$).

To conclude the proof of Theorem \ref{thm:big}, it remains to show (\ref{localaic}) leads to the claimed convergence. An argument similar to Lemma \ref{lem:aic1toconverge} works, but since (\ref{localaic}) only holds locally over $\mathbb{B}_2^d(\theta^*;\frac{c}{\|\theta^*\|_2})$, there are two differences of (i) requiring a warm initialization and (ii) requiring a sample complexity slightly higher than (\ref{samcom2}). 
\begin{lem}[(\ref{localaic}) implies convergence] Under (\ref{localaic}) and  
\begin{align}
    \label{localsam}
    n\gtrsim \|\theta^*\|_2^7d~~~\textrm{with large enough hidden constant}
\end{align}
$\{\theta_0\}_{t\ge 0}$ generated by Algorithm \ref{alg:gd} with $\theta_0\in\mathbb{B}_2^d(\theta^*;\frac{c_0}{\|\theta^*\|_2})$ (and the $\eta$ in (\ref{localaic})) satisfies
\begin{align*}
    \|\theta_t-\theta^*\|_2\le \bigg(1-\frac{c'}{\|\theta^*\|_2^2}\bigg)^t\frac{c_0}{\|\theta^*\|_2} + \tilde{C}\sqrt{\frac{\|\theta^*\|_2^5d}{n}},\quad \forall t=0,1,\cdots.
\end{align*}
\end{lem}
\begin{proof}
     We introduce a sequence $\{f_t\}_{t\ge 0}$ by the recurrence and the initial value
     \begin{align}
         f_{t+1} = \bigg(1-\frac{c'}{\|\theta^*\|_2^2}\bigg)f_t+ C'\bigg(\sqrt{\frac{\|\theta^*\|_2d}{n}}+\frac{\|\theta^*\|_2d}{n}\bigg),\quad t\ge 0,\quad\textrm{and}\quad f_0=\frac{c_0}{\|\theta^*\|_2}.\label{defft}
     \end{align}
     One can find the closed-form expression
     \begin{align}\label{closedft1}
    f_t = \left( 1 - \frac{c'}{\|\theta^*\|_2^2} \right)^t \frac{c_0}{\|\theta^*\|_2} + \left[ 1 - \left( 1 - \frac{c'}{\|\theta^*\|_2^2} \right)^t \right] \frac{C'}{c'} \left( \sqrt{\frac{\|\theta^*\|_2^5 d}{n}} + \frac{\|\theta^*\|_2^3 d}{n} \right),\quad t\ge 0.
\end{align}
Under (\ref{localsam}), we have $\frac{\|\theta^*\|_2^3d}{n}\le\sqrt{\frac{\|\theta^*\|_2^5d}{n}}$, hence it is sufficient to prove $\|\theta_t-\theta^*\|_2\le f_t$ for any $t\ge 0.$ Also under (\ref{localsam}) we have 
\[\frac{C'}{c'} \left( \sqrt{\frac{\|\theta^*\|_2^5 d}{n}} + \frac{\|\theta^*\|_2^3 d}{n} \right) \le \frac{c_0}{\|\theta^*\|_2},\]
and as such (\ref{closedft1}) gives $f_t\le \frac{c_0}{\|\theta^*\|_2}$ for any $t\ge 0$.

    All that remains is to prove $\|\theta_t-\theta^*\|_2\le f_t$, $t\ge 0$, and we shall achieve this by induction. Due to $\theta_0\in \mathbb{B}_2^d(\theta^*;\frac{c_0}{\|\theta^*\|_2})$, this holds trivially for $t=0$. Suppose that $\|\theta_t-\theta^*\|_2\le f_t$, we then have $\|\theta_t-\theta^*\|_2\le \frac{c_0}{\|\theta^*\|_2}$, thus $\theta_t\in \mathbb{B}_2^d(\theta^*;\frac{c_0}{\|\theta^*\|_2})$. Therefore we can set $u=\theta_t$ in (\ref{localaic}), yielding 
    \[ \|\theta_t-\theta^*-\eta\cdot h_{\theta^*}(\theta_t)\|_2 \le \bigg(1-\frac{c'}{\|\theta^*\|_2^2}\bigg)\|\theta_t-\theta^*\|_2 +C'\bigg(\sqrt{\frac{\|\theta^*\|_2d}{n}}+\frac{\|\theta^*\|_2d}{n}\bigg).\]
    By $\theta_{t+1}=\theta_t-\eta\cdot h_{\theta^*}(\theta_t)$ and $\|\theta_t-\theta^*\|_2\le f _t$, the above display gives
    \begin{align*}
        \|\theta_{t+1}-\theta^*\|_2\le  \bigg(1-\frac{c'}{\|\theta^*\|_2^2}\bigg)f_t +C'\bigg(\sqrt{\frac{\|\theta^*\|_2d}{n}}+\frac{\|\theta^*\|_2d}{n}\bigg)\stackrel{(\ref{defft})}{=} f_{t+1},
    \end{align*}
    thus completing the induction. The claim is proved. 
\end{proof}

 \section{Proof of Theorem \ref{thm:initial} (Performance of Algorithm \ref{alg:spectral})}\label{app:proofsharper}
 In this proof, for simplicity in notation, we denote $\hat{r}_{T_0}$ by $\hat{r}$ and in turn $\hat{\theta}=\hat{\lambda}\cdot\hat{r}$. We start with a decomposition 
\begin{align} \nn
    \|\hat{\theta}-\theta^*\|_2 &\le \bigg\|\hat{\lambda}\hat{r} -\|\theta^*\|_2\hat{r}\bigg\|_2 + \bigg\|\|\theta^*\|_2\hat{r}-\theta^*\bigg\|_2 \\&= \big|\hat{\lambda}-\|\theta^*\|_2\big| + \|\theta^*\|_2\bigg\|\hat{r}-\frac{\theta^*}{\|\theta^*\|_2}\bigg\|_2,\label{inierror2terms}
\end{align}
where the first term captures the norm estimation error while the second term accounts for the direction estimation error.

Due to Lemma \ref{lem:directionbound}, we immediately obtain the bound on $\|\theta^*\|_2\big\|\hat{r}-\frac{\theta^*}{\|\theta^*\|_2}\big\|_2$. In fact, by assuming the bound in Lemma \ref{lem:directionbound} holds (this is within the   probability promised  in Theorem \ref{thm:initial}), then 
we have
\begin{align}
    \label{directiontermbound}
    \|\theta^*\|_2\bigg\|\hat{r}-\frac{\theta^*}{\|\theta^*\|_2}\bigg\|_2 \lesssim  \polylog(n) \bigg(\sqrt{\frac{\|\theta^*\|_2d}{n }}+\frac{\|\theta^*\|_2d}{n}\bigg).
\end{align}
In the remainder of this appendix, we focus on bounding the norm estimation error, $|\hat{\lambda}-\|\theta^*\|_2|$.

\subsection{Bounding $|\hat{\lambda}-\|\theta^*\|_2|$ (norm estimation error)}
In order to bound $|\hat{\lambda}-\|\theta^*\|_2|=\Big|q^{-1}(\hat{v}^\top \hat{r})-\|\theta^*\|_2\Big|,$
we shall start with bounding
\[\Big|\hat{v}^\top \hat{r}-q(\|\theta^*\|_2)\Big|.\]
\begin{lem}
    \label{lem:firstbound} If $n\gtrsim d$ and $\|\theta^*\|_2\ge 1$, then it holds with probability at least $1-n^{-2}$  that 
    \begin{align}\label{nonqbound}
        \Big|\hat{v}^\top \hat{r}-q(\|\theta^*\|_2)\Big| \le \polylog(n) \bigg(\frac{1}{\sqrt{n}}+\frac{d}{n\|\theta^*\|_2}+\frac{d^2}{n^2}\bigg) .  
    \end{align}
\end{lem}
\begin{proof}
We shall start with the decomposition 
\begin{align*}
    &\big|\langle\hat{v},\hat{r}\rangle - q(\|\theta^*\|_2)\big|\le \underbrace{\bigg|\bigg\langle\hat{v},\frac{\theta^*}{\|\theta^*\|_2}\bigg\rangle-q(\|\theta^*\|_2)\bigg|  }_{I_1}+\underbrace{\bigg|\bigg\langle\hat{v},\hat{r}-\frac{\theta^*}{\|\theta^*\|_2}\bigg\rangle\bigg|}_{:=I_2}.
\end{align*}
\paragraph{Bounding $I_1.$} Due to $\hat{v}=\frac{1}{(1-\nu)n}\sum_{i=\nu n+1}^ny_ix_i$ and (\ref{Ehatvre}) we can write  
 \begin{align} \nn
    I_1:&= \bigg|\frac{1}{(1-\nu)n}\sum_{i=\nu n+1}^n y_i\bigg\langle x_i,\frac{\theta^*}{\|\theta^*\|_2}\bigg\rangle-\mathbb{E}\Big(y_i\Big\langle x_i,\frac{\theta^*}{\|\theta^*\|_2}\Big\rangle\Big)\bigg|. 
\end{align}
By $|y_i|\le 1$, \[\Big\|y_i\Big\langle x_i, \frac{\theta^*}{\|\theta^*\|_2}\Big\rangle\Big\|_{\psi_2}\lesssim \Big\|\Big\langle x_i, \frac{\theta^*}{\|\theta^*\|_2}\Big\rangle\Big\|_{\psi_2}\lesssim 1,\]
and by centering \cite[Lemma 2.6.8]{vershynin2018high}, 
 \begin{align}
     \Big\|y_i\Big\langle x_i, \frac{\theta^*}{\|\theta^*\|_2}\Big\rangle - \mathbb{E}\Big(y_i\Big\langle x_i, \frac{\theta^*}{\|\theta^*\|_2}\Big\rangle\Big)\Big\|_{\psi_2} \lesssim 1,\label{simi1}
 \end{align}
and moreover by property of sum of independent, zero-mean sub-Gaussian variables \cite[Proposition 2.6.1]{vershynin2018high},
 \begin{align}
     \Big\|\frac{1}{(1-\nu)n}\sum_{i=\nu n+1}^ny_i\Big\langle x_i, \frac{\theta^*}{\|\theta^*\|_2}\Big\rangle - \mathbb{E}\Big(y_i\Big\langle x_i, \frac{\theta^*}{\|\theta^*\|_2}\Big\rangle\Big)\Big\|_{\psi_2} \lesssim \frac{1}{\sqrt{n}}.\label{simi2}
 \end{align}
 In turn, a standard sub-Gaussian tail bounds yields that 
 \begin{align}\label{I1boundnow}
     \mathbb{P}\bigg(I_1\lesssim \sqrt{\frac{\log n}{n}}\bigg)\ge 1-\frac{1}{n^2}.
 \end{align} 

\paragraph{Bounding $I_2$.} We first substitute $\hat{v}$ to write 
\begin{align*}
    I_2 = \bigg|\frac{1}{(1-\nu)n}\sum_{i=\nu n+1}^n y_i\Big\langle x_i,\hat{r}-\frac{\theta^*}{\|\theta^*\|_2}\Big\rangle\bigg|
\end{align*}
We now let $e = \hat{r}-\frac{\theta^*}{\|\theta^*\|_2}$
and decompose it along $\frac{\theta^*}{\|\theta^*\|_2}$ and orthogonal space, 
\begin{align}\nn
    e = \Big\langle e,\frac{\theta^*}{\|\theta^*\|_2}\Big\rangle\frac{\theta^*}{\|\theta^*\|_2} + \underbrace{e- \Big\langle e,\frac{\theta^*}{\|\theta^*\|_2}\Big\rangle\frac{\theta^*}{\|\theta^*\|_2}}_{e_{\theta^*}^\perp}.
\end{align}
Therefore, 
\begin{align}\nn
    &I_2 = \bigg|\frac{1}{(1-\nu)n}\sum_{i=\nu n+1}^n y_i\langle x_i,e\rangle\bigg| \\&\le \underbrace{\bigg|\frac{1}{(1-\nu)n}\sum_{i=\nu n}^ny_i\Big\langle x_i,\frac{\theta^*}{\|\theta^*\|_2}\Big\rangle\bigg|\bigg|\Big\langle e,\frac{\theta^*}{\|\theta^*\|_2}\Big\rangle\bigg|}_{:=I_{21}} + \underbrace{\bigg|\frac{1}{(1-\nu)n}\sum_{i=\nu n+1}^n y_i \langle x_i, e_{\theta^*}^\perp\rangle\bigg|}_{:=I_{22}}.\label{I1decompose}
\end{align}
To bound $I_{11}$, observe that 
\begin{align*}
    \bigg|\Big\langle e,\frac{\theta^*}{\|\theta^*\|_2}\Big\rangle\bigg| =\bigg|\Big\langle \hat{r},\frac{\theta^*}{\|\theta^*\|_2}\Big\rangle -1\bigg| =\frac{\|\hat{r}-\frac{\theta^*}{\|\theta^*\|_2}\|_2^2}{2}
\end{align*}
Also, by (\ref{I1boundnow}) and Lemma \ref{lem:mono}, 
\begin{align*}
    \bigg|\frac{1}{(1-\nu)n}\sum_{i=\nu n+1}^n y_i\bigg\langle x_i,\frac{\theta^*}{\|\theta^*\|_2}\bigg\rangle\bigg|\le q(\|\theta^*\|_2) +O\bigg(\sqrt{\frac{\log n}{n}}\bigg) \lesssim 1.
\end{align*}
Therefore, by the direction error bound in Lemma \ref{lem:directionbound}, 
\begin{align}\label{I11boudnow}
    I_{21}\lesssim \bigg\|\hat{r}-\frac{\theta^*}{\|\theta^*\|_2}\bigg\|_2^2 \le \polylog(n) \max\bigg\{\frac{d}{n\|\theta^*\|_2},\frac{d^2}{n^2}\bigg\}. 
\end{align}
To bound $I_{12}$, notice that it is zero-mean and satisfies   
\begin{align*}
    \bigg\|\frac{1}{(1-\nu)n}\sum_{i=\nu n+1}^ny_i\langle x_i,e^\perp_{\theta^*}\rangle\bigg\|_{\psi_2} = \bigg\|\bigg\langle \frac{1}{(1-\nu)n} \sum_{i=\nu n+1}^n \big\{y_ix_i-\mathbb{E}(y_ix_i)\big\}, e^\perp_{\theta^*}\bigg\rangle\bigg\|_{\psi_2}\lesssim \frac{\|e\|_2}{\sqrt{n}}
\end{align*}
by a reasoning similar to Equations (\ref{simi1})--(\ref{simi2}). Combining with the   bound on $\|e\|_2$ in Lemma \ref{lem:directionbound}, we reach
\[ \bigg\|\frac{1}{(1-\nu)n}\sum_{i=\nu n+1}^ny_i\langle x_i,e^\perp_{\theta^*}\rangle\bigg\|_{\psi_2} \lesssim \sqrt{\frac{\polylog n}{n}}\bigg(\sqrt{\frac{d}{n\|\theta^*\|_2}}+\frac{d}{n}\bigg).\]
Hence, by a sub-Gaussian tail bound similarly to (\ref{I1boundnow}),
\begin{align}\label{I12boundnow}
    I_{22}\lesssim\sqrt{\frac{\polylog n}{n}}\max\bigg\{\sqrt{\frac{d}{n\|\theta^*\|_2}},\frac{d}{n}\bigg\}
\end{align}
holds with probability at least $1-\frac{1}{n^2}$. Substituting (\ref{I11boudnow}) and (\ref{I12boundnow}) into (\ref{I1decompose}) yields
\begin{align}
    \label{I1finalbound11}
    I_2 \le \polylog(n)\bigg(\frac{\sqrt{d}}{n\sqrt{\|\theta^*\|_2}}+\frac{d}{n^{3/2}}+\frac{d^2}{n^2}+\frac{d}{n\|\theta^*\|_2}\bigg)
\end{align}
Combining (\ref{I1boundnow}) and (\ref{I1finalbound11}), under $n\gtrsim d$ and $\|\theta^*\|_2\ge 1$, we obtain 
\begin{align*}
    \big|\langle\hat{v},\hat{r}\rangle-q(\|\theta^*\|_2)\big|\le \polylog(n) \bigg(\frac{1}{\sqrt{n}}+\frac{d}{n\|\theta^*\|_2}+\frac{d^2}{n^2}\bigg) 
\end{align*}    
as claimed. 
\end{proof}
 
 We now examine the effect of applying $q^{-1}$ and seek to  bound $|\hat{\lambda}-\|\theta^*\|_2| = |q^{-1}(\hat{v}^\top \hat{r})-q^{-1}(q(\|\theta^*\|_2))|.$
\begin{lem}
    \label{lem:finalnormbound} Assume that the event (\ref{nonqbound}) in Lemma \ref{lem:firstbound} holds, i.e.,
    \begin{align}
          \Big|\hat{v}^\top \hat{r}-q(\|\theta^*\|_2)\Big| \le \polylog(n) \bigg(\frac{1}{\sqrt{n}}+\frac{d}{n\|\theta^*\|_2}+\frac{d^2}{n^2}\bigg) . \label{nonqbouncopy}
    \end{align}Under
    \begin{align}\label{samnownow}
        n\ge \polylog(n)\bigg(\|\theta^*\|_2^4+\|\theta^*\|_2d\bigg)
    \end{align} we have 
    \begin{align}
        |\hat{\lambda}-\|\theta^*\|_2|\le \polylog (n)\bigg(\|\theta^*\|_2^3 \sqrt{\frac{1}{n}}+ \frac{\|\theta^*\|_2^2d}{n}\bigg)\label{normerrorbbb}
    \end{align}
\end{lem}
\begin{proof}
   Letting $q(\infty)=\frac{1}{\sqrt{2\pi}}$, a key fact in this proof is the order of the asymptotic gap (cf. Lemma \ref{lem:orderofgap}) 
\begin{align}\label{gapeq}
    G(\tau):=q(\infty)-q(\tau)\asymp \frac{1}{\tau^2},\qquad\forall \tau\ge \frac{1}{2},
\end{align}
which implies
\begin{align}
    q(\infty) - q(\|\theta^*\|_2)\asymp\frac{1}{\|\theta^*\|_2^2}.\label{gaptheta}
\end{align}On the other hand, (\ref{samnownow}) and (\ref{nonqbouncopy}) together ensure that 
\begin{align}\label{gaptheta1}
    |\hat{v}^\top \hat{r}-q(\|\theta^*\|_2)|\le \frac{c_*}{\|\theta^*\|_2^2}\quad\textrm{for some small enough $c_*.$}
\end{align}
Combining (\ref{gaptheta}) and (\ref{gaptheta1}) gives 
\begin{align}\nn 
    \hat{v}^\top \hat{r}< q(\infty)- \frac{c_1}{\|\theta^*\|_2^2}\quad\textrm{for some universal constant $c_1,$}
\end{align}
hence it is meaningful to apply $q^{-1}$ to $\hat{v}^\top \hat{r}$.

Now we write the right-hand side of (\ref{nonqbouncopy}) as $E$, which satisfies $E\le\frac{c_*}{\|\theta^*\|_2^2}$, then we have \[q(\|\theta^*\|_2)-E\le \hat{v}^\top \hat{r}\le q(\|\theta^*\|_2)+E \]and moreover
\begin{align}\nn
    &|\hat{\lambda}-\|\theta^*\|_2|=\big|q^{-1}(\hat{v}^\top \hat{r})-q^{-1}(q(\|\theta^*\|_2))\big| \\\label{twotobound}
    &\le \max\left\{q^{-1}\Big(q(\|\theta^*\|_2)+E\Big)-\|\theta^*\|_2,\|\theta^*\|_2-q^{-1}\Big(q(\|\theta^*\|_2)-E\Big)\right\}\\\explain \textrm{by the monotonicity of $q$ and $q^{-1}$ (cf. Lemma \ref{lem:mono})}
\end{align}
We shall first bound
\begin{align*}
    \hat{\lambda}_{\max}:= q^{-1}\Big(q(\|\theta^*\|_2)+E\Big)
\end{align*}
which satisfies \begin{align}
    q(\hat{
\lambda
}_{\max})= q(\|\theta^*\|_2)+E\label{whatisqmax}
\end{align} The asymptotic gap is given by 
\begin{align}\label{121}
   G(\hat{\lambda}_{\max})= q(\infty) - q(\hat{\lambda}_{\max})= q(\infty) - q(\|\theta^*\|_2)-E = G(\|\theta^*\|_2)-E.
\end{align}
 Now using (\ref{gapeq}) and $E\le\frac{c_*}{\|\theta^*\|_2^2}$ for some small enough $c_*$, 
\begin{align}\label{uplambdamax}
    &\frac{1}{\hat{\lambda}_{\max}^2}\asymp G(\hat{\lambda}_{\max}) = G(\|\theta^*\|_2)-E\gtrsim \frac{1}{\|\theta^*\|_2^2}
   ~\Longrightarrow~ \hat{\lambda}_{\max}\lesssim \|\theta^*\|_2
\end{align}
By the mean value theorem and the derivative for an inverse function,  
\begin{align*}
    &q^{-1}\Big(q(\|\theta^*\|_2)+E\Big)-\|\theta^*\|_2\\
    &=q^{-1}\Big(q(\|\theta^*\|_2)+E\Big)-q^{-1}\big(q(\|\theta^*\|_2)\big)\\
    &= q^{-1}(q(\hat{\lambda}_{\max}))-q^{-1}\big(q(\|\theta^*\|_2)\big)\\\explain\textrm{by Equation (\ref{whatisqmax})}\\
    &\le (q^{-1})'(\xi_0)\big(q(\hat{\lambda}_{\max})-q(\|\theta^*\|_2)\big)\\\explain\textrm{by mean value theorem, $\xi_0\in[q(\|\theta^*\|_2),q(\hat{\lambda}_{\max})]$}\\
    &= \frac{E}{q'(\hat{\xi}_0)}\\\explain\textrm{by derivative of inverse function, $\hat{\xi}_0\in [\|\theta^*\|_2,\hat{\lambda}_{\max}]$, and (\ref{whatisqmax})}\\
    &\lesssim \|\theta^*\|_2^3E.\\\explain\textrm{By (\ref{uplambdamax}) and Lemma \ref{lem:separa}}
\end{align*}
We now similarly bound the second term in (\ref{twotobound}),  
\begin{align*}
    &\|\theta^*\|_2-q^{-1}\Big(q(\|\theta^*\|_2)-E\Big)\\
    &= q^{-1}(q(\|\theta^*\|_2))-q^{-1}\Big(q(\|\theta^*\|_2)-E\Big)\\
    &=  (q^{-1})'(\xi_1)\cdot E\\\explain\textrm{by mean value theorem, $\xi_1\in[q(\|\theta^*\|_2)-E,q(\|\theta^*\|_2)]$}\\
    & = \frac{1}{q'(\hat{\xi}_1)}E\\\explain\textrm{by derivative of inverse function, $\hat{\xi}_1<\|\theta^*\|_2$}\\
    & \lesssim  \|\theta^*\|_2^3E.\\\explain \textrm{By Lemma \ref{lem:separa}}
\end{align*}
Substituting the past two displays and $E=\polylog(n) \big (\frac{1}{\sqrt{n}}+\frac{d}{n\|\theta^*\|_2}+\frac{d^2}{n^2}\big ) $ into (\ref{twotobound}) gives
\begin{align}\nn 
    |\hat{\lambda}-\|\theta^*\|_2|\le \polylog(n)\bigg(\frac{\|\theta^*\|_2^3}{\sqrt{n}}+\frac{\|\theta^*\|_2^2d}{n}+\frac{\|\theta^*\|_2^3d^2}{n^2}\bigg)
\end{align}which together with $\frac{\|\theta^*\|_2^3d^2}{n^2}\le \frac{\|\theta^*\|_2^2d}{n}$ under $n\gtrsim\|\theta^*\|_2d$ concludes the proof. 
\end{proof}

 We now substitute (\ref{directiontermbound}) and (\ref{normerrorbbb}) into  (\ref{inierror2terms}) to complete the proof of Theorem \ref{thm:initial}.
 
 \section{Technical Lemmas}\label{app:lemma}
 \begin{lem}[Stein identities \cite{stein1972bound,stein2004use}]\label{lem:stein}
Let $g\sim N(0,1)$ and let $\phi:\mathbb{R}\to\mathbb{R}$ be absolutely continuous with
$
\mathbb E|\phi'(g)|<\infty,~
\mathbb E|g\phi(g)|<\infty.
$ Then
\begin{align}
    \mathbb{E}[g\phi(g)]=\mathbb{E}[\phi'(g)].\label{stein1}
\end{align}

Moreover, let $x\sim N(0,I_d)$ and let $\Phi:\mathbb{R}^d\to\mathbb{R}$ be continuously differentiable with
$
\mathbb E\|\nabla\Phi(x)\|_2<\infty,
~
\mathbb E\|x\Phi(x)\|_2<\infty.
$
Then
\begin{align}
    \mathbb{E}[x\Phi(x)]=\mathbb{E}[\nabla\Phi(x)].\label{stein2}
\end{align}
In particular, for $\Phi(x)=f(x^\top\theta)$,
\begin{align}
    \mathbb{E}[x f(x^\top\theta)]
=
\mathbb{E}[f'(x^\top\theta)]\theta,\label{stein3}
\end{align}
provided both sides are well-defined.
\end{lem} 
\begin{lem}
    \cite[Theorem 2.10]{13concen} \label{lem:bernstein210} Let $X_1,...,X_n$ be independent random variables, and assume that for some $v,c>0$, 
    \begin{gather*}
        \sum_{i=1}^n \mathbb{E}X_i^2 \le v,\\
        \sum_{i=1}^n \mathbb{E}|X_i|^q\le \frac{q!}{2}vc^{q-2},\quad \textrm{for all integers }q\ge 3,
    \end{gather*}
    then we have
    \begin{align*}
        \mathbb{P}\bigg(\bigg|\sum_{i=1}^n(X_i-\mathbb{E}X_i)\bigg|\ge \sqrt{2vt}+ct\bigg) \le 2\exp(-t),\quad \forall t>0. 
    \end{align*}
\end{lem}
  \begin{lem}\label{bern} \cite[Theorem 2.8.1]{vershynin2018high} Let $X_1,...,X_N$ be independent, zero-mean, sub-exponential random variables. Then for every $t\geq 0$, for some absolute constant $c$ we have \begin{equation}
        \nonumber\mathbb{P}\left(\Big|\sum_{i=1}^NX_i\Big|\geq t\right)\leq 2\exp\left(-c\min\Big\{\frac{t^2}{\sum_{i=1}^N \|X_i\|_{\psi_1}^2},\frac{t}{\max_{1\leq i\leq N}\|X_i\|_{\psi_1}}\Big\}\right)
    \end{equation}
 \end{lem}
 \begin{lem}\cite[Theorem 7.3.1 \& Corollary 7.3.3]{vershynin2018high}
\label{lem:gaubound} Let $X\in R^{n\times d}$ be a matrix with independent $N(0,1)$ entries. Then $\mathbb{E}\|X\|_{op} \le \sqrt{n}+\sqrt{d}$ and 
\[\mathbb{P}\bigg(\|X\|_{op}\ge \sqrt{n}+\sqrt{d}+t\bigg)  \le 2\exp(-ct^2),\quad \forall t\ge 0.\]
\end{lem}

\subsection{Functions  Arising in the Proofs}
\begin{lem}
    \label{lem:derivative} 
    For $s(a)=\frac{1}{1+e^{-a}}$, we have 
   \[s'(a) = \frac{e^{-a}}{(1+e^{-a})^2}=s(a)\big(1-s(a)\big)\]
   and $\sup_{a\in\mathbb{R}}|s'(a)|\le \frac{1}{4}$. 
\end{lem}
\begin{proof}
    It can be shown by direct differentiation. 
\end{proof}

\begin{lem}[Order of $m(\tau)$ and $q'(\tau)$]\label{lem:separa}
Let $g\sim N(0,1)$ and
\begin{gather*}
    m(\tau)=\mathbb{E} [s'(\tau g)],\\
    q(\tau)=\tau m(\tau)=\mathbb{E}[\tau s'(\tau g)]=\mathbb{E}[s(\tau g)g]. 
\end{gather*} 
Then, for all $\tau\ge 0$,
\[
m(\tau)\asymp \frac{1}{1+\tau},
\qquad
q'(\tau)=\mathbb{E}[s'(\tau g)g^2]\asymp \frac{1}{(1+\tau)^3}.
\]
\end{lem}

\begin{proof}
We first note that
$
s'(a)=\frac{1}{(e^{a/2}+e^{-a/2})^2}
,
$
hence there exist universal constants $c,C>0$ such that  
\begin{align}
    c e^{-|a|}\le s'(a)\le C e^{-|a|},\qquad\forall a\in\mathbb{R}.\label{sprimaorder}
\end{align}

\paragraph{Proof of $m(\tau)\asymp\frac{1}{1+\tau}$.}
In view of (\ref{sprimaorder}),
\[
m(\tau)=\mathbb{E}[s'(\tau g)]
\asymp
\mathbb{E}[e^{-\tau|g|}].
\]
To show $m(\tau)\asymp\frac{1}{1+\tau}$, all that remains is to establish 
$
\mathbb{E}[e^{-\tau|g|}]\asymp \frac{1}{1+\tau}.$
To show this, note that 
\[
\mathbb{E}[e^{-\tau|g|}]
=
\frac{2}{\sqrt{2\pi}}\int_0^\infty e^{-\tau x}e^{-x^2/2}\,dx.
\]
For the upper bound, since $e^{-x^2/2}\le 1$,
\[
\int_0^\infty e^{-\tau x}e^{-x^2/2}\,dx
\le
\int_0^\infty e^{-\tau x}\,dx
=
\frac{1}{\tau}
\]
for $\tau>0$. While for $0\le \tau\le 1$ the integral is bounded by an absolute constant. Taken collectively, 
\[
\mathbb{E}[e^{-\tau|g|}]\lesssim \frac{1}{1+\tau}.
\]
For the lower bound, integrate over $0\le x\le (1+\tau)^{-1}$. On this interval, we have
\[
\tau x\le 1,
\qquad
x^2/2\le 1/2,
\]
and in turn $
e^{-\tau x}e^{-x^2/2}\ge e^{-3/2}$, which yields 
\[
\int_0^\infty e^{-\tau x}e^{-x^2/2}\,dx
\ge
\int_0^{(1+\tau)^{-1}} e^{-\tau x}e^{-x^2/2}\,dx
\gtrsim
\frac{1}{1+\tau}.
\]
This proves
$
m(\tau)\asymp \frac{1}{1+\tau}.$

\paragraph{Proof of $q'(\tau)\asymp\frac{1}{(1+\tau)^3}$.}
 By Stein's identity (cf. Equation (\ref{stein1}) of Lemma \ref{lem:stein}),
\[
q(\tau)=\tau\mathbb{E}[s'(\tau g)]
=
\mathbb{E}[s(\tau g)g],
\]
and differentiating gives $q'(\tau)=\mathbb{E}[s'(\tau g)g^2].$
Using again $s'(a)\asymp e^{-|a|}$ from (\ref{sprimaorder}), we get
\[
q'(\tau)
\asymp
\mathbb{E}[g^2 e^{-\tau|g|}].
\]
Now
\[
\mathbb{E}[g^2 e^{-\tau|g|}]
=
\frac{2}{\sqrt{2\pi}}
\int_0^\infty x^2 e^{-\tau x}e^{-x^2/2}\,dx.
\]

For the upper bound, since $e^{-x^2/2}\le 1$,
\[
\int_0^\infty x^2 e^{-\tau x}e^{-x^2/2}\,dx
\le
\int_0^\infty x^2 e^{-\tau x}\,dx
=
\frac{2}{\tau^3}
\]
for $\tau>0$. For $0\le \tau\le 1$, the expectation is bounded by an absolute constant. Hence
\[
\mathbb{E}[g^2 e^{-\tau|g|}]
\lesssim
\frac{1}{(1+\tau)^3}.
\]

For the lower bound, we shall integrate over $0\le x\le (1+\tau)^{-1}$. On this interval,
we have $
e^{-\tau x}e^{-x^2/2}\ge e^{-3/2}.
$
Thus
\[
\int_0^\infty x^2 e^{-\tau x}e^{-x^2/2}\,dx
\ge
\int_0^{(1+\tau)^{-1}} x^2 e^{-\tau x}e^{-x^2/2}\,dx
\gtrsim
\int_0^{(1+\tau)^{-1}} x^2\,dx
\asymp
\frac{1}{(1+\tau)^3}.
\]
Therefore
\[
q'(\tau)=\mathbb{E}[s'(\tau g)g^2]\asymp \frac{1}{(1+\tau)^3}.
\]
This completes the proof.
\end{proof}

\begin{lem}[Comparing $m(\tau)$ and $q'(\tau)$]\label{lem:V1larger} Define $m(\tau)$ and $q'(\tau)$ as in Lemma \ref{lem:separa}.  
    We have
    \[m(\tau)> q'(\tau),\quad\forall \tau>0.\]
\end{lem}

\begin{proof}
In view of Lemma \ref{lem:derivative}, $s'$ is an even function and is strictly decreasing in $|a|$. 
Let $X:=g^2.$
Then $X\ge 0$ and $\mathbb E X=1$. For each fixed $\tau\ge 0$, define
\[
\varphi_\tau(x):=s'(\tau\sqrt{x}),\qquad x\ge 0.
\]
Since $s'(a)$ is decreasing in $|a|$, the function $\varphi_\tau$ is nonincreasing on $[0,\infty)$. Moreover, if $\tau>0$, then $\varphi_\tau$ is strictly decreasing. With this notation,
\[
m(\tau)=\mathbb E[\varphi_\tau(X)],
\qquad
q'(\tau)=\mathbb E[X\varphi_{\tau}(X)].
\]
Therefore
\begin{align}
    q'(\tau)-m(\tau)
=
\mathbb E[X\varphi_{\tau}(X)]-\mathbb E[X]\mathbb E[\varphi_{\tau}(X)]
=
\operatorname{Cov}(X,\varphi_{\tau}(X)).\label{cfh1recall}
\end{align}
We claim this covariance is not greater than $0$. Indeed, let $X'$ be an independent copy of $X$. Then
\[
\operatorname{Cov}(X,\varphi_{\tau}(X))
=
\frac12\mathbb E\Big[(X-X')(\varphi_{\tau}(X)-\varphi_{\tau}(X'))\Big].
\]
Since $\varphi_{\tau}$ is nonincreasing, we have
$
(X-X')(\varphi_{\tau}(X)-\varphi_{\tau}(X'))\le 0
$, which yields
\[
\operatorname{Cov}(X,\varphi_{\tau}(X))\le 0.
\]
In light of (\ref{cfh1recall}), 
\[
q'(\tau)-m(\tau)\le 0,
\]
which gives $
m(\tau)\ge q'(\tau).$ If $\tau>0$, then $\varphi_{\tau}$ is strictly decreasing and $X=g^2$ is nondegenerate, so the covariance is strictly negative. Therefore
\[ 
m(\tau)>q'(\tau),\qquad\forall \tau>0,
\] 
as claimed. 
\end{proof} 
\begin{lem}[Upper bound on $\mathbb{E}(s'(\tau g)|g|^q)$]\label{lem:Hq-upper}
For every integer $q\ge 3$ and every $\tau\ge 1$,
\[
\mathbb{E}_{g\sim N(0,1)}\big[s'(\tau g)|g|^q\big]\le \frac{2q!}{\sqrt{2\pi}}\frac{1}{\tau^{q+1}}.
\] 
\end{lem}
\begin{proof}
Let $H_q(\tau):=\mathbb{E}_{g\sim N(0,1)}\big[s'(\tau g)|g|^q\big]$ for convenience. In light of (\ref{sprimaorder}), 
\[
H_q(\tau)
\le
\mathbb{E}\big[e^{-\tau|g|}|g|^q\big]=\frac{2}{\sqrt{2\pi}}
\int_0^\infty e^{-\tau x}x^q e^{-x^2/2}\,dx.
\] 
Since \(e^{-x^2/2}\le 1\), we have
\[
\mathbb{E}\big[e^{-\tau|g|}|g|^q\big]
\le
\frac{2}{\sqrt{2\pi}}
\int_0^\infty e^{-\tau x}x^q\,dx.
\]
By the change of variables \(u=\tau x\),
\[
\int_0^\infty e^{-\tau x}x^q\,dx
=
\frac{1}{\tau^{q+1}}\int_0^\infty e^{-u}u^q\,du
=
\frac{\Gamma(q+1)}{\tau^{q+1}}
=
\frac{q!}{\tau^{q+1}}.
\]
Therefore,
\[
H_q(\tau)
\le
\frac{2q!}{\sqrt{2\pi}}\frac{1}{\tau^{q+1}}.
\]
This proves the claim.
\end{proof}
\begin{lem}[Map $F:\mathbb{R}^d\to\mathbb{R}^d$ in Equation (\ref{T2expression})] \label{lem:integral}Let $F(z)=m(\|z\|_2)z$ for $z\in\mathbb{R}^d$ where $m(\tau)= \mathbb{E}_{g\sim (0,1)}(s'(\tau g))$. Then we have: 
\begin{itemize}
    \item The Jacobian of $F$ at $z\ne 0$ is given by \begin{align}\label{jaco}
        \nabla F(z)=m(\|z\|_2)I_d+ \frac{m'(\|z\|_2)}{\|z\|_2}zz^\top . 
    \end{align} 

    \item With the convention $\nabla F(0):=m(0)I_d$, we have the integral representation 
    \begin{align}\label{integral}
        F(u)-F(v) = \int_0^1 \nabla F(v+t h)h\, dt
    \end{align}
    where $h=u-v$. 
\end{itemize}
\end{lem}
 
\begin{proof}
 We first compute the Jacobian (\ref{jaco}). Write
$
\rho:=\|z\|_2=\left(\sum_{k=1}^d z_k^2\right)^{1/2}.$
The $i$-th component of $F$ is
$
F_i(z)=m(\rho)z_i.$
For $z\neq 0$, we have
$
\frac{\partial \rho}{\partial z_j}
=
\frac{z_j}{\rho}.$ Therefore, by the product rule and the chain rule,
\[
\frac{\partial F_i(z)}{\partial z_j}
=
\frac{\partial}{\partial z_j}\big(m(\rho)z_i\big)
=
m'(\rho)\frac{\partial \rho}{\partial z_j}z_i
+
m(\rho)\frac{\partial z_i}{\partial z_j}=m'(\rho)\frac{z_j}{\rho}z_i
+
m(\rho)\mathbf 1_{\{i=j\}}.
\] 
Thus the $(i,j)$-th entry of the Jacobian matrix is
\[
(\nabla F(z))_{ij}
=
m(\rho)\mathbf 1_{\{i=j\}}
+
\frac{m'(\rho)}{\rho}z_i z_j.
\]
In matrix form, this is exactly
\[
\nabla F(z)
=
m(\rho)I_d+\frac{m'(\rho)}{\rho}zz^\top
=
m(\|z\|_2)I_d
+
\frac{m'(\|z\|_2)}{\|z\|_2}zz^\top.
\]

We next prove the integral representation (\ref{integral}). Let
$
h:=u-v
$ and define the line segment
$\gamma(t):=v+th,~~t\in[0,1].$ We shall assume that $\gamma(t)\ne 0$ for $t\in[0,1]$; we can cover the remaining cases by a separate treatment (or an approximation argument) and the convention $\nabla F(0)=m(0)I_d$ does not affect the value of the integral.  For each coordinate $i=1,\ldots,d$, define the one-dimensional function
$
\psi_i(t):=F_i(\gamma(t))=F_i(v+th).
$ 
By the chain rule,
\[
\psi_i'(t)
=
\sum_{j=1}^d
\frac{\partial F_i(\gamma(t))}{\partial z_j}h_j.
\]
This is the $i$-th component of the vector
$\nabla F(\gamma(t))h.$ Hence, by the fundamental theorem of calculus,
\[
F_i(u)-F_i(v)
=
\psi_i(1)-\psi_i(0)
=
\int_0^1 \psi_i'(t)\,dt
=
\int_0^1
\left(\nabla F(v+th)h\right)_i
\,dt.
\]
Since this holds for every coordinate $i$, we obtain the vector identity
in (\ref{integral}).
This completes the proof.
\end{proof}
\begin{lem}[Upper bound on $m'(\tau)$]
\label{mporder}
For
$
m(\tau):=\mathbb E_{g\sim N(0,1)}[s'(\tau g)],
$
we have
\[
|m'(\tau)|\lesssim \frac{1}{\tau^2},
\qquad \tau\ge \frac12.
\]
\end{lem}

\begin{proof} 
Differentiating under the expectation gives
\[
m'(\tau)=\mathbb E[s''(\tau g)g].
\]
We now use the explicit formula for $s''$. Since
$
s'(a)
=
\frac{1}{(e^{a/2}+e^{-a/2})^2},
$
we have
\[
s''(a)
=
-\frac{e^{a/2}-e^{-a/2}}{(e^{a/2}+e^{-a/2})^3}
=
\frac{e^{-a/2}-e^{a/2}}{(e^{a/2}+e^{-a/2})^3}.
\]
Therefore,
\[
|s''(a)|
=
\frac{|e^{a/2}-e^{-a/2}|}{(e^{a/2}+e^{-a/2})^3}.
\]
If $a\ge 0$, then
\[
|s''(a)|
\le
\frac{e^{a/2}}{e^{3a/2}}
=
e^{-a}.
\]
If $a<0$, then
\[
|s''(a)|
\le
\frac{e^{-a/2}}{e^{-3a/2}}
=
e^{a}
=
e^{-|a|}.
\]
Hence, for all $a\in\mathbb R$,
\begin{align}
    |s''(a)|\le e^{-|a|}.\label{haveproved}
\end{align}

It follows that, for $\tau\ge 1/2$,
\[
|m'(\tau)|
\le
\mathbb E[|s''(\tau g)||g|]
\le
\mathbb E[e^{-\tau|g|}|g|].
\]
Using the Gaussian density,
\[
\mathbb E[e^{-\tau|g|}|g|]
=
\frac{2}{\sqrt{2\pi}}
\int_0^\infty x e^{-\tau x}e^{-x^2/2}\,dx.
\]
Since $e^{-x^2/2}\le 1$, we obtain
\[
\mathbb E[e^{-\tau|g|}|g|]
\le
C\int_0^\infty x e^{-\tau x}\,dx.
\]
Finally,
\[
\int_0^\infty x e^{-\tau x}\,dx
=
\frac{1}{\tau^2}\int_0^\infty u e^{-u}\,du
=
\frac{1}{\tau^2},
\]
where we used the change of variables $u=\tau x$. Therefore,
\[
|m'(\tau)|\le \frac{C}{\tau^2},
\qquad \tau\ge \frac12.
\]
This proves the claim.
\end{proof}
\begin{lem}[Upper bound on $q''(\tau)$]
\label{qpporder}
For $q(\tau)$ defined in Lemma~\ref{lem:separa}, we have
\[
|q''(\tau)|\lesssim \frac{1}{\tau^4},
\qquad \tau\ge \frac12.
\]
\end{lem}

\begin{proof} 
Recall that
$
q'(\tau)=\mathbb E[s'(\tau g)g^2].
$ 
Differentiating under the expectation gives
\[
q''(\tau)=\mathbb E[s''(\tau g)g^3].
\] 
Recall from Equation (\ref{haveproved}) in the proof of Theorem \ref{mporder} that, for all $a\in\mathbb R$,
\[
|s''(a)|\le e^{-|a|}.
\]

It follows that, for $\tau\ge 1/2$,
\[
|q''(\tau)|
\le
\mathbb E[|s''(\tau g)||g|^3]
\le
\mathbb E[e^{-\tau|g|}|g|^3].
\]
Using the Gaussian density,
\[
\mathbb E[e^{-\tau|g|}|g|^3]
=
\frac{2}{\sqrt{2\pi}}
\int_0^\infty x^3 e^{-\tau x}e^{-x^2/2}\,dx.
\]
Since $e^{-x^2/2}\le 1$, we get
\[
\mathbb E[e^{-\tau|g|}|g|^3]
\le
C\int_0^\infty x^3 e^{-\tau x}\,dx.
\]
Finally,
\[
\int_0^\infty x^3e^{-\tau x}\,dx
=
\frac{1}{\tau^4}\int_0^\infty u^3e^{-u}\,du
=
\frac{6}{\tau^4},
\]
where we used the change of variables $u=\tau x$. Therefore,
\[
|q''(\tau)|\le \frac{C}{\tau^4},
\qquad \tau\ge \frac12.
\]
This proves the claim.
\end{proof}

\begin{lem}\label{lem:mono}
 Let $q(\tau)$ be defined as in Lemma \ref{lem:separa}. $q(\tau)$ is strictly increasing on $[0,\infty)$ and satisfies
\[
q(0)=0,
\qquad
\lim_{\tau\to\infty}q(\tau)=\frac{1}{\sqrt{2\pi}}.
\]
Consequently, $q$ admits an inverse on
$
(0,\frac{1}{\sqrt{2\pi}}),$ which we denote by $q^{-1}:(0,\frac{1}{\sqrt{2\pi}})\to (0,\infty)$.
\end{lem}

\begin{proof}
By Lemma \ref{lem:separa}, for every $\tau\ge 0$,
\[
q(\tau)=\tau\mathbb E[s'(\tau g)]
=
\mathbb E[s(\tau g)g].
\]
We first prove monotonicity. Differentiating $q(\tau)=\mathbb E[s(\tau g)g]$ gives
\[
q'(\tau)
=
\mathbb E[s'(\tau g)g^2].
\]
Since
$
s'(a)=\frac{1}{(e^{a/2}+e^{-a/2})^2}>0
$ 
for every $a\in\mathbb R$, and since $\mathbb P(g\neq 0)=1$, we have
\[
q'(\tau)=\mathbb E[s'(\tau g)g^2]>0
\]
for every $\tau\ge 0$. Hence $q$ is strictly increasing on $[0,\infty)$.

Next,
\[
q(0)=0\cdot m(0)=0.
\]
For the limit as $\tau\to\infty$, using again the representation
\[
q(\tau)=\mathbb E[s(\tau g)g],
\]
we note that for every $g\neq 0$,
\[
s(\tau g)\to \mathbf 1_{\{g>0\}}
\qquad\text{as ~}\tau\to\infty.
\]
Moreover,
we have $
|s(\tau g)g|\le |g|
$
and $\mathbb E|g|<\infty$. Thus, by dominated convergence,
\[
\lim_{\tau\to\infty}q(\tau)
=
\mathbb E[g\mathbf 1_{\{g>0\}}]=
\frac{1}{\sqrt{2\pi}}\int_0^\infty x e^{-x^2/2}\,dx
=
\frac{1}{\sqrt{2\pi}}.
\] Since $q$ is continuous and strictly increasing from $0$ to $1/\sqrt{2\pi}$, it admits an inverse on
$\left(0,\frac{1}{\sqrt{2\pi}}\right).$
\end{proof}
  \begin{lem}\label{lem:orderofgap}
 Let $q(\tau)$ be defined as in Lemma \ref{lem:separa}.  For every $\tau\ge 1/2$,
\[
\frac{1}{\sqrt{2\pi}}-q(\tau)\asymp \frac{1}{\tau^2}.
\] 
\end{lem}

\begin{proof}
In light of Lemma \ref{lem:separa}, 
\[
q(\tau)=\tau\mathbb E[s'(\tau g)]
=
\mathbb E[s(\tau g)g].
\]
Therefore,
\[
\frac{1}{\sqrt{2\pi}}-q(\tau)
=
\mathbb E[g\mathbf 1_{\{g>0\}}]-\mathbb E[s(\tau g)g].
\]
Using the symmetry of the standard Gaussian distribution, we can write
\begin{align*}
\frac{1}{\sqrt{2\pi}}-q(\tau)
&=
\int_0^\infty x\phi(x)\,dx
-
\int_{\mathbb R} s(\tau x)x\phi(x)\,dx \\
&=
\int_0^\infty x\phi(x)\,dx
-
\int_0^\infty x\phi(x)\big(s(\tau x)-s(-\tau x)\big)\,dx,
\end{align*}
where
$$
\phi(x)=\frac{1}{\sqrt{2\pi}}e^{-x^2/2}
$$ is the density function of $N(0,1).$
Since
$
s(a)-s(-a)=2s(a)-1,
$ 
we get
\[
1-\big(s(\tau x)-s(-\tau x)\big)
=
1-(2s(\tau x)-1)
=
2(1-s(\tau x))
=
\frac{2}{1+e^{\tau x}}.
\]
Hence
\[
\frac{1}{\sqrt{2\pi}}-q(\tau)
=
2\int_0^\infty \frac{x\phi(x)}{1+e^{\tau x}}\,dx.
\]
Changing variables $u=\tau x$ gives
\[
\frac{1}{\sqrt{2\pi}}-q(\tau)
=
\frac{2}{\tau^2}
\int_0^\infty
\frac{u\phi(u/\tau)}{1+e^u}\,du.
\]

We now prove matching upper and lower bounds. For the upper bound, since
\[
\phi(u/\tau)\le \phi(0)=\frac{1}{\sqrt{2\pi}},
\]
we have
\[
\frac{1}{\sqrt{2\pi}}-q(\tau)
\le
\frac{2\phi(0)}{\tau^2}
\int_0^\infty \frac{u}{1+e^u}\,du
\le
\frac{C}{\tau^2}.
\]

For the lower bound, restrict the integral to $u\in[0,1]$. Since $\tau\ge 1/2$,
we have $u/\tau\le 2$ for $u\in[0,1]$. Thus
\[
\phi(u/\tau)\ge \phi(2)>0
\qquad\text{for }0\le u\le 1.
\]
Therefore,
\[
\frac{1}{\sqrt{2\pi}}-q(\tau)
\ge
\frac{2\phi(2)}{\tau^2}
\int_0^1 \frac{u}{1+e^u}\,du
\ge
\frac{c}{\tau^2}.
\]
Combining the two bounds proves the claim.
\end{proof}
\begin{lem}\label{Ehatv}
Under Assumptions \ref{gaussian}--\ref{normge1}, we have 
\[
v=\mathbb E[y_i x_i]
=
q(\|\theta^*\|_2)\frac{\theta^*}{\|\theta^*\|_2},
\]
where $q(\tau):=\mathbb E_{g\sim N(0,1)}[s(\tau g)g].$ Note that by Stein's identity $q(\tau)=\mathbb{E}[\tau s'(\tau g)]$ and thus coincides with the definition in Lemma \ref{lem:separa}. 
\end{lem}

\begin{proof}
By conditioning on $x_i$,
\[
\mathbb E[y_i x_i]
=
\mathbb E[s(x_i^\top\theta^*)x_i].
\]
If $\theta^*=0$, the claim is immediate since the expectation is zero. Assume
$\theta^*\neq 0$ and set
$
e:=\frac{\theta^*}{\|\theta^*\|_2}.
$
By rotational invariance of the standard Gaussian distribution,
$
\mathbb E[s(x_i^\top\theta^*)x_i]
$ 
must be parallel to $\theta^*$, and its component along $e$ is
\[
\left\langle
\mathbb E[s(x_i^\top\theta^*)x_i],e
\right\rangle
=
\mathbb E[s(\|\theta^*\|_2 g)g]
=
q(\|\theta^*\|_2),
\]
where $g=x_i^\top e\sim N(0,1)$. Hence
$
\mathbb E[y_i x_i]
=
q(\|\theta^*\|_2)e
=
q(\|\theta^*\|_2)\frac{\theta^*}{\|\theta^*\|_2}.
$ 
\end{proof}

\section{Estimator in \cite{matsumoto2025learning} and Algorithm \ref{alg:spectral}}\label{app:relu}
In this small appendix we provide more details on the estimator of \cite{matsumoto2025learning}, in order to support Equation (\ref{lloydequa}) and Lemma \ref{lem:directionbound} in the main text.

\subsection{Sub-Gradient Descent}\label{app:subgd}
We first explain that Equation (\ref{lloydequa}), or equivalently
\begin{gather}\label{subgd}
    \tilde{r}_{t+1} =\hat{r}_t-\frac{\sqrt{2\pi}}{n}\sum_{i=1}^{n}\big(\frac{\sign(x_i^\top \hat{r}_t)+1}{2}-y_i\big)x_i,\\\hat{r}_{t+1}=\frac{\tilde{r}_{t+1}}{\|\tilde{r}_{t+1}\|_2} \label{normalize1}
\end{gather}
as written in Lemma \ref{lem:directionbound}, can be viewed as a (sub-)gradient descent algorithm with respect to the ReLU loss. To see this, we transition to the $\{-1,1\}$-valued responses $\tilde{y}_i = 2y_i-1$, 
consider the Hamming distance  loss \[L_{\rm ham}(u)=\frac{1}{n}\sum_{i=1}^n \mathbf{1}\big(\tilde{y_i}\ne \sign(x_i^\top  u)\big)=\frac{1}{n}\sum_{i=1}^n \mathbf{1}\big(-\tilde{y}_ix_i^\top u \ge 0\big)\] and then consider its convex relaxation, the ReLU loss obtained by replacing $\mathbf{1}(a>0)$ to ${\rm ReLU}(a)=\max\{a,0\}=\frac{a+|a|}{2}$,  
\[L_{\rm ReLU}(u)=\frac{1}{2n}\sum_{i=1}^n \Big(|x_i^\top u|-\tilde{y}_ix_i^\top u\Big).\]
It is easy to see that a sub-gradient of the ReLU loss is given by 
\[\partial L_{\rm ReLU}(u) = \frac{1}{2n}\sum_{i=1}^n\Big(\sign(x_i^\top u)-\tilde{y}_i\Big)x_i = \frac{1}{n}\sum_{i=1}^n \Big(\frac{\sign(x_i^\top u)+1}{2}-y_i\Big)x_i .\]
As such, (\ref{subgd}) is exactly the sub-gradient descent with ReLU loss. Moreover, the additional normalization step, (\ref{normalize1}), is natural since the desired true direction $\theta^*/\|\theta^*\|_2$ resides in $\mathbb{S}^{d-1}$.

\subsection{Derivation of Lemma \ref{lem:directionbound}}\label{app:derilem}
We show the derivation of  Lemma \ref{lem:directionbound} from Corollary 7 and Theorem 5 of \cite{matsumoto2025learning} with the sparsity level $k$ therein set as $d$. To start, we note that $\beta$ in \cite{matsumoto2025learning} denotes the signal norm $\|\theta^*\|_2$. Hence $\beta \ge 1$ under Assumption \ref{normge1}.

We first show that (\ref{directionerror}) follows from the error rate of \cite[Corollary 7]{matsumoto2025learning}. In particular, Equation (17) therein shows that, to achieve $\epsilon$ estimation error, we require a sample size of 
\begin{gather}
    n=\tilde{\Omega}\bigg(\frac{d}{\epsilon^2\|\theta^*\|_2}\bigg),\qquad\textrm{if ~}1\le\|\theta^*\|_2\le \frac{C_0}{\epsilon};\\
    n=\tilde{\Omega}\bigg(\frac{d}{\epsilon}\bigg),\qquad\textrm{if ~}\|\theta^*\|_2> \frac{C_0}{\epsilon}.
\end{gather}
This implies an error rate of $$\epsilon_0
=\tilde{O}\bigg(\sqrt{\frac{d}{n\|\theta^*\|_2}}+\frac{d}{n}\bigg)$$
under any $\|\theta^*\|_2\ge 1$, as with our Equation  (\ref{directionerror}).

We then show the probability of $1-e^{-n}$ and the iteration complexity of $\log_2\log_2\frac{n}{d}$, as stated in our Lemma \ref{lem:directionbound}. In fact, we can set $\rho$ of \cite[Theorem 5]{matsumoto2025learning}, the probability parameter therein, as $e^{-n}$. In view of Equation (14) therein, this can only add factor of $\polylog(n)$ to the sample complexity, which remains at $n\gtrsim d\polylog n$, as assumed in our Lemma \ref{lem:directionbound}. Moreover, in view of Equation (16) of \cite{matsumoto2025learning}, the sequence of direction estimates $\{\hat{r}_t\}_{t\ge 0}$ achieves a very fast convergence rate,
\[\bigg\|\hat{r}_t-\frac{\theta^*}{\|\theta^*\|_2}\bigg\|_2 \le 2^{2^{-t}}\epsilon_0^{1-2^{-t}},\qquad\forall t\ge 0.\]
A direct calculation then finds that $\hat{r}_t$ with $t\ge \log_2\log_2\frac{n}{d}$ already achieves the desired direction estimation error of $\tilde{O}(\epsilon_0)$. 

\subsection{Numerical Performance of Algorithm \ref{alg:spectral}}\label{app:numalg2}
We compare the estimation performance of Algorithm \ref{alg:spectral} with GD under the following settings:
\begin{itemize}
    \item We set $T_0=30$ in Algorithm \ref{alg:spectral} and run it without sample splitting --- i.e., both Step 1 and Step 2 use all the $n$ samples.  

    \item We run GD in Algorithm \ref{alg:gd} with $\theta_0=0$, $\eta=4$ for $400$ iterations. Suppose that the MLE exists, then traditional optimization theory \cite{nesterov2018lectures} guarantees that the GD iterates converge to MLE, and as we run sufficiently many iterations we expect that the estimator, $\theta_{400}$, well approximates MLE. 
\end{itemize}
Consistent with Remark \ref{rem:improve}, our numerical simulations suggest that Algorithm \ref{alg:spectral} outperforms GD, and possibly also MLE, in high-dimensional setting, especially when $n$ is only moderately large. We provide two numerical examples to illustrate this phenomenon:
\begin{itemize}
    \item We first test both algorithms in a high-dimensional regime with $(d,\|\theta^*\|_2)=(1000,2)$ and $n\in\{3000,5000,8000,10000,15000,20000,30000\}$. The log-log plots in Figure~\ref{fig:sub11} suggest that Algorithm~\ref{alg:spectral} yields more accurate estimates than GD when $n<10000$, while GD becomes slightly better when $n>10000$. The reference dashed line further confirms the $O(n^{-1/2})$ error decay rate of Algorithm~\ref{alg:spectral}.
    
    \item We next consider the setting $(n,d)=(5000,1000)$ and vary $\|\theta^*\|_2$ over $\{1,2,4,6,8\}$. The results in Figure~\ref{fig:sub21} again show that Algorithm~\ref{alg:spectral} can substantially outperform GD in high-dimensional, moderate-sample regimes.
\end{itemize}
\begin{figure}[ht!]
    \centering

    \begin{subfigure}[t]{0.49\textwidth}
        \centering
        \includegraphics[width=\textwidth]{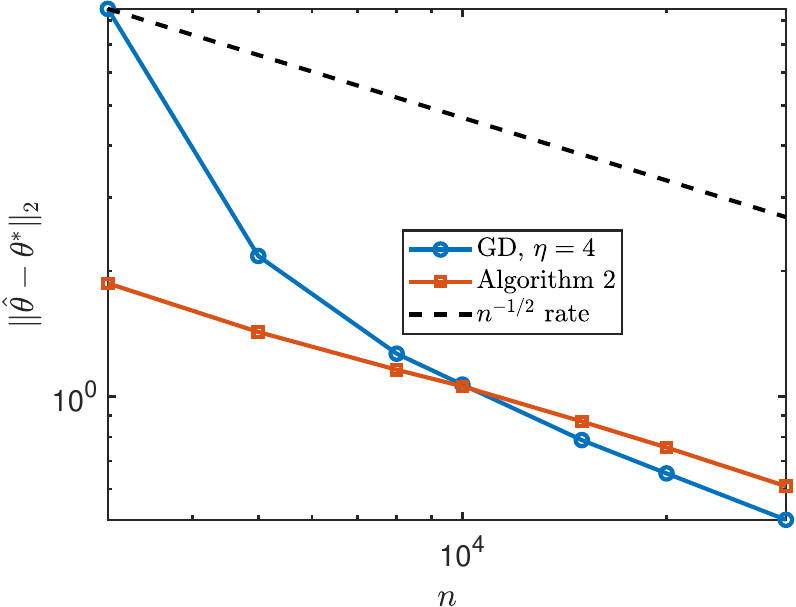}
        \caption{Error v.s. $n$}
        \label{fig:sub11}
    \end{subfigure}
    \hfill
    \begin{subfigure}[t]{0.46\textwidth}
        \centering
        \includegraphics[width=\textwidth]{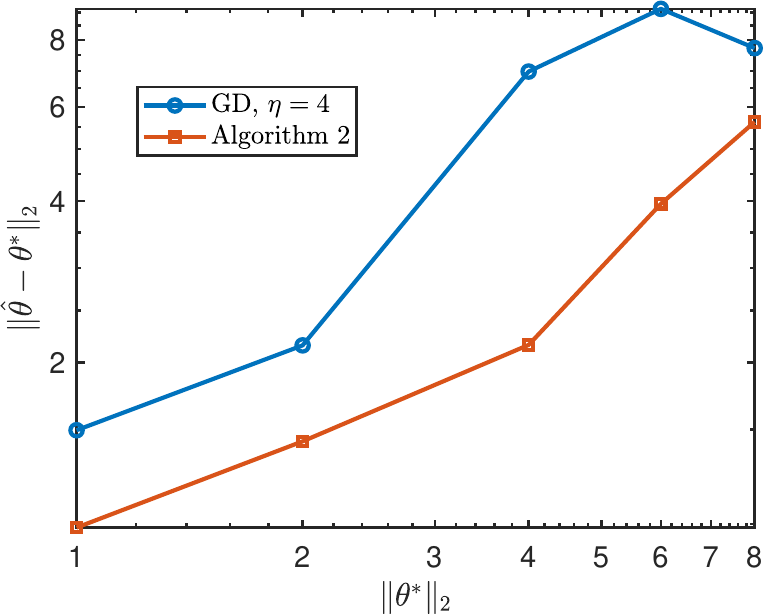}
        \caption{Error v.s. $\|\theta^*\|_2$}
        \label{fig:sub21}
    \end{subfigure}

    \caption{Comparing Algorithm \ref{alg:spectral} with small stepsize GD.}
    \label{fig:two_figs}
\end{figure}

\end{document}